\documentclass[12pt,final]{mycolt2015} 

\usepackage{bbold,mathrsfs} 
\usepackage{enumitem}
\setlist[enumerate]{topsep=1pt,itemsep=-0.5pt,parsep=2pt}
\usepackage{thm-restate}
\usepackage[toc,page]{appendix}

\newtheorem{thm}{Theorem}[section]
\newtheorem{cor}{Corollary}[section]
\newtheorem{lem}{Lemma}[section]
\newtheorem{prop}{Proposition}[section]
\newtheorem{assumption}{Assumption}[section]

\newtheorem{rem}{Remark}[section]

\numberwithin{equation}{section}

\newenvironment{proofof}[1]{\par\noindent{\bfseries\upshape
  Proof of #1\ }}{\hfill\BlackBox\\*[-0.2mm]}

\def\re{\mathbb{R}} 
 \def\rn{\mathbb{R}^n}
\def\e{e}
\def\F{F}
\def\fe{\phi}
\def\Fe{\Phi}
\def\i{i}
\def\bi{d_{\pi^o,\i}}
\def\1{\mathbf{1}}
\def\0{\mathbf{0}}
\def\w{\theta}
\def\Z{Z}
\def\M{M}
\def\bM{{\bar M}}
\def\P{P_\pi}
\def\PL{P_{\pi,\gamma}^\lambda}
\def\Pb{P_{\pi^o}}
\def\Po{\mathbf{P}^o}
\def\Gm{\Gamma}
\def\Lm{\Lambda}
\def\r{r_\pi}
\def\rl{r_{\pi,\gamma}^\lambda}
\def\d{d_{\pi^o}}

\def\bb{b}
\def\bA{C}

\def\G{G}
\def\C{L}
\def\wtld{\widetilde}

\def\E{\mathbb{E}}

\def\A{\mathcal{A}}
\def\S{\mathcal{S}}
\def\H{B}
\def\asto{\overset{a.s.}{\to}}

\def\J{\mathcal{J}}
\def\mF{G}

\def\mC{H}
\def\nr{\mathbf{r}}
\def\la{($\lambda$)\xspace}

\title{On Convergence of Emphatic Temporal-Difference Learning%
\thanks{This research was supported by a grant from Alberta Innovates--Technology Futures. A shorter (28-page) version of the first version of this paper appeared at \emph{the 28th Annual Conference on Learning Theory} (COLT), Paris, France, 2015.}%
\thanks{Here are the corrections made since the first version (the conclusions of the paper have not been affected): $\qquad \qquad$ \hfill
(i)~In the second version we corrected an oversight in a proof in the first version: the original Prop.~C.1 needs to be proved based on the proof of the first part of Prop.~C.2. Corrections were made in the statement of Prop.~C.1, the last paragraph of the proof of Prop.~C.2, and the references to these two propositions in Appendix C and Section 2.2.\hfill
(ii)~In this third version we rewrote Footnote 17 on page 28 to correct a proof argument used in Appendix A.4 of the second version.
}}
\usepackage{times}

\coltauthor{\Name{Huizhen Yu} \Email{janey.hzyu@gmail.com}\\
\addr Department of Computing Science, University of Alberta, Canada
}

\begin{document}

\maketitle

\begin{abstract}
We consider emphatic temporal-difference learning algorithms for policy evaluation in discounted Markov decision processes with finite spaces.
Such algorithms were recently proposed by \citet*{SuMW14} as an improved solution to the problem of divergence of off-policy temporal-difference learning with linear function approximation. We present in this paper the first convergence proofs for two emphatic algorithms, ETD($\lambda$) and ELSTD($\lambda$). We prove, under general off-policy conditions, the convergence in $L^1$ for ELSTD($\lambda$) iterates, and the almost sure convergence of the approximate value functions calculated by both algorithms using a single infinitely long trajectory. Our analysis involves new techniques with applications beyond emphatic algorithms leading, for example, to the first proof that standard TD($\lambda$) also converges under off-policy training for $\lambda$ sufficiently large.
\end{abstract}

\bigskip
\bigskip

\clearpage
\tableofcontents

\clearpage

\section{Introduction} \label{sec-intro}
We consider discounted finite-spaces Markov decision processes (MDPs) and the problem of learning an approximate value function for a given policy from \emph{off-policy} data, that is, from data due to a different policy. The first policy is called the \emph{target} policy and the second is called the \emph{behavior} policy. For example, one may want to learn value functions for many target policies in parallel from one (exploratory) behavior; this requires off-policy learning.

We focus on temporal-difference (TD) methods with linear function approximation \citep{Sut88}.
Such methods are typically convergent when the target and behavior policies are the same (the \emph{on-policy} case), but not in the off-policy case~\citep{tr-disc}.
This difficulty is intrinsic to sampling states according to an arbitrary policy.\footnote{See the papers \citep{Baird95,tr-disc,SuMW14} and the books \citep{BET,SuB} for related examples and discussion.} Gradient-based or least squares-based approaches have been used to avoid this difficulty.\footnote{See e.g., \citep{maei11,by08,bruno14,dnp14}.}

Recently, Sutton, Mahmood, and White \citeyearpar{SuMW14} proposed a new approach to address this issue more directly. They introduced an \emph{emphatic TD\la} algorithm, or \emph{ETD\la} as we call it here.  
The approach is related to the early work on episodic off-policy TD($\lambda$) \citep{offpolicytd-psd}, and is based on the idea of re-weighting the states when forming the eligibility traces in TD($\lambda$), so that the weights reflect the occupation frequencies of the target policy rather than the behavior policy. The result of this weighting scheme is that the ``mean updates'' associated with ETD($\lambda$) now involve a negative definite matrix, similar to the convergent on-policy TD algorithms. This is a salient feature of the emphatic TD method.

The purpose of this paper is to investigate the convergence properties of ETD($\lambda$) and its least-squares version, ELSTD($\lambda$).
Under general conditions, we show that (see Theorems~\ref{thm-lstd},~\ref{thm-td-unconstrained}):
\begin{enumerate}
\item[(i)] for stepsizes decreasing as $t^{-c}, c \in (1/2, 1]$, the matrix and vector iterates generated by ELSTD($\lambda$) converge in $L^1$ to the desired limits, which define a projected Bellman equation;
\item[(ii)] for stepsizes decreasing as $t^{-1}$, both algorithms generate approximate value functions that converge almost surely to the desired solution of an associated projected Bellman equation.
\end{enumerate}
These results show that the new emphatic TD algorithms are sound for off-policy learning.

Regarding proof techniques, we note that although the ``mean updates'' of ETD($\lambda$) involve a negative definite matrix, it is still difficult to directly apply results from stochastic approximation theory to establish rigorously the association between the ``mean updates'' and the ETD($\lambda$) iterates, thereby obtaining the desired convergence. The stability criterion of \citep{BorM00} (see also \citep[Chap.~3]{Bor08}) and the ``natural averaging'' argument in \citep[Chap.~6]{Bor08} seem suitable, but they require a certain tightness condition that is hard to verify in the general off-policy learning setting where the variances of the trace iterates can grow to infinity with time.\footnote{Related examples can be found in \citep{gi-sampling,rj-tdimportance,SuMW14}.}
The analysis of \citep{tr-disc} has a strong condition (Condition (6), p.\ 683, in particular), which is difficult to satisfy unless the trace iterates are uniformly bounded. But in general, this would impose a strong restriction on the behavior policy~\cite[cf.][Prop.\ 3.1, Footnote 3, and the discussion in p.\ 3320-3322]{Yu-siam-lstd}.

For regular off-policy LSTD($\lambda$) and TD($\lambda$) \citep{by08}, it has been shown by \citet{Yu-siam-lstd} that the associated joint process of states and trace iterates exhibit useful properties, by which convergence results for LSTD($\lambda$) can be derived. Subsequently, the results can be used to furnish the conditions of a convergence theorem from stochastic approximation theory \citep{KuY03} and yield convergence results for TD($\lambda$). In this paper we will take the proof approach used in \citep{Yu-siam-lstd}. We note, however, that most of the intermediate results needed in our case require different and more involved proofs, due to the complexity of the emphatic TD method. Furthermore, we will give a new argument to prove the almost sure convergence of ETD($\lambda$), which applies also to the regular off-policy TD($\lambda$) of \citep{by08} for $\lambda$ near $1$. This improves a result of \citep{Yu-siam-lstd}, which only dealt with a constrained version of TD($\lambda$) that restricts the iterates to lie in a bounded set.

This paper is organized as follows. In Section~\ref{sec-alg} we formulate the approximate policy evaluation problem, and we describe the ETD($\lambda)$ and ELSTD($\lambda$) algorithms, and the approximate Bellman equations they aim to solve. We also state our main convergence results in this section. In Section~\ref{sec-elstd} we prove our convergence theorem for ELSTD($\lambda$), and prepare results needed for analyzing ETD($\lambda$) with a ``mean ODE''
\footnote{ODE stands for ordinary differential equation.} 
method. In Section~\ref{sec-etd} we prove our convergence theorem for ETD($\lambda$).
We collect long proofs, technical lemmas and other related results in Appendices~\ref{appsec-a}-\ref{appsec-ndef}.

\section{Emphatic TD Algorithms: ETD($\lambda$) and ELSTD($\lambda$)} \label{sec-alg}

\subsection{A Policy Evaluation Problem in Off-Policy Learning}
Let $\S = \{ 1, \ldots, N\}$ be the state space, and let $\A$ be a finite set of actions. We assume, without loss of generality, that for every state, all actions are feasible. If we take action $a \in \A$ at state $s \in \S$, the system moves from state $s$ to state $s'$ with probability $p(s' \!\mid s, a)$, and we receive a random reward with mean $r(s,a,s')$ and bounded variance, according to a probability distribution $q(\cdot \mid s, a, s')$.

We are interested in evaluating the performance of a given stationary policy
\footnote{A \emph{stationary policy} is a decision rule that specifies the probability of taking action $a$ at state $s$ for every $s \in \S$.}
$\pi$, the target policy, without knowledge of the MDP model.
The evaluation is to be done by using just observations of state transitions and rewards, while following a stationary policy $\pi^o \not= \pi$, the behavior policy.

Starting from time $t=0$, applying $\pi$ would generate a sequence of rewards $R_0, R_1, \ldots$. 
The performance of $\pi$ will be measured in terms of the expected total rewards attained under $\pi$ up to a random termination time $\tau \geq 1$ that depends on the states in a Markovian way.
In particular, if at time $t \geq 1$, the state is $s$ and termination has not occurred yet, then the probability of $\tau = t$ (terminating at time $t$) is $1 - \gamma(s)$, for a given parameter $\gamma(s) \in [0,1]$.

Let $\P$ denote the transition matrix of the Markov chain on $\S$ induced by $\pi$.  Let $\Gm$ denote the $N \times N$ diagonal matrix with diagonal entries $\gamma(s), s \in \S$. Denote by $\pi(a \!\mid s)$ and $\pi^o(a \!\mid s)$ the probability of taking action $a$ at state $s$ under the policy $\pi$ and $\pi^o$, respectively.

\begin{assumption}[Conditions on the target and behavior policies] \label{cond-bpolicy} \hfill
\begin{enumerate}
\item[{\rm (i)}] The target policy $\pi$ is such that $(I - \P \Gm)^{-1}$ exists (equivalently, termination occurs with probability $1$ under $\pi$, for any initial state).
\item[{\rm (ii)}] The behavior policy $\pi^o$ induces an irreducible Markov chain on $\S$, and moreover, for all $(s,a) \in \mathcal{S} \times \mathcal{A}$, $\pi^o(a \!\mid s) > 0$ if $\pi(a \!\mid s) > 0$.
\end{enumerate}
\end{assumption}

Under Assumption~\ref{cond-bpolicy}(i), we define the \emph{value function} of the target policy $\pi$ by $v_\pi: \S \to \re$,
$v_{\pi}(s) = \E_\pi \left[ \, \sum_{t=0}^{\tau-1} R_t  \, \Big| \, S_0 = s \right]$,
where $\E_{\pi}$ denotes expectation with respect to the probability distribution of the process of states, actions and rewards, $(S_t, A_t, R_t)$, $t \geq 0$, induced by the policy $\pi$.
Let $r_{\pi}$ be the expected one-stage reward function under $\pi$; i.e., $r_{\pi}(s) = \E_\pi \big[ R_0 \mid S_0=s \big]$ for $s \in \S$. 
Then the desired function $v_{\pi}$ can be seen to satisfy uniquely the Bellman equation
\footnote{One can verify this Bellman equation directly. It also follows from the standard MDP theory \citep[see e.g.,][]{puterman94}, as by definition $v_\pi$ here can be related to a value function in a discounted MDP where the discount factors depend on state transitions, similar to discounted semi-Markov decision processes.}
$$v_\pi = r_{\pi} + \P \Gm \, v_\pi, \qquad \text{i.e.}, \quad v_\pi = (I - \P \Gm)^{-1} r_\pi.$$

\subsection{Algorithms}

We consider computing $v_\pi$ with the ETD($\lambda$) algorithm \citep{SuMW14} and its least-squares version, ELSTD$(\lambda)$, 
using linear function approximation, while following the behavior policy $\pi^o$. 
Let $E \subset \re^N$ be the approximation subspace of dimension $n$, and let $\Fe$ be an $N \times n$ matrix whose columns form a basis of $E$. 
The approximation problem is to find a parameter vector $\w \in \rn$ such that $v = \Fe \w \in E$ approximates $v_{\pi}$ well. 

We express $v = \Fe\w$ as $v(s) = \fe(s)^\top \w$, $s \in \S$, where the superscript $\text{}^\top$ stands for transpose, and $\fe(s) \in \rn$ is the transposed $s$-th row of $\Fe$ and represents the ``features'' of state $s$. 
Like standard TD($\lambda$), if a transition $(s,s')$ occurs with reward $r'$, ETD($\lambda$) and ELSTD$(\lambda)$ use the ``temporal difference'' term, $r' + \gamma(s') \fe(s')^\top \w - \fe(s)^\top \w$,
to adjust the parameter $\w$ for the approximate value function. Also like standard TD($\lambda$), these algorithms aim to solve a projected (single-step or multistep) Bellman equation; but we shall defer the discussion of this until after describing the ETD($\lambda$) algorithm.

We focus on a general form of the ETD($\lambda$) algorithm, which uses state-dependent $\lambda$ values specified by a function $\lambda: \S \to [0,1]$. 
Inputs to the algorithm are the states, actions and rewards, $\{(S_t, A_t, R_{t}), t\geq 0 \}$, generated under the behavior policy $\pi^o$,
where $R_{t}$ is the random reward received upon the transition from state $S_t$ to $S_{t+1}$ with action $A_t$.
The algorithm can access the following functions, in addition to the features $\fe(s)$: 
\begin{enumerate}
\item[(i)] $\gamma: \S \to [0,1]$, which specifies the termination probabilities (or equivalently, the state-dependent discount factors) that define $v_{\pi}$, as described earlier;
\item[(ii)] $\lambda: \S \to [0,1]$, which determines the single or multi-step Bellman equation for the algorithm [cf.\ the subsequent Eqs.~(\ref{eq-bellman})-(\ref{eq-bellman-def})]; 
\item[(iii)] $\rho: \mathcal{S} \times \mathcal{A} \to \re_+$ given by $\rho(s, a) = \pi(a \!\mid s)/ \pi^o(a \!\mid s)$ (with $0/0=0$),  which gives the likelihood ratios for action probabilities that can be used to compensate for sampling states and actions according to the behavior policy $\pi^o$ instead of the target policy $\pi$;
\item[(iv)] $\i: \mathcal{S} \to \re_+$, which gives the algorithm additional flexibility to weigh states according to the degree of ``interest'' indicated by $\i(s)$.
\end{enumerate}

The ETD($\lambda$) algorithm does the following. For each $t \geq 0$, let $\alpha_t \in (0,1]$ be a stepsize parameter, and to simplify notation, let
$$ \rho_t = \rho(S_t, A_t), \qquad \gamma_t = \gamma(S_t), \qquad \lambda_t = \lambda(S_t).   $$
ETD($\lambda$) calculates recursively $\w_t \in \rn$, $t \geq 0$, according to
\begin{equation} \label{eq-emtd0}
  \w_{t+1} = \w_t + \alpha_t \, \e_t \cdot \rho_t \, \big( R_{t} + \gamma_{t+1} \fe(S_{t+1})^\top \w_t - \fe(S_t)^\top \w_t \big),
\end{equation}
where $\e_t \in \rn$ (called the ``eligibility trace'') is calculated together with two nonnegative scalar iterates $(\F_t, \M_t)$ according to:
\footnote{For insights about ETD($\lambda$), see \citep{SuMW14,MYWS15}. Our definition (\ref{eq-td1}) of $\{\e_t\}$ differs slightly from its original definition, but the two are equivalent; ours appears to be more convenient for our analysis.}
\begin{align}
   \F_t & = \gamma_t \, \rho_{t-1}  \,  \F_{t-1} + \i(S_t),    \label{eq-td3} \\
   \M_t & = \lambda_t \, \i(S_t) + ( 1 - \lambda_t ) \, \F_t, \label{eq-td2} \\
   \e_t & =  \lambda_t \, \gamma_t \,  \rho_{t-1} \, \e_{t-1} + \M_t \, \fe(S_t). \label{eq-td1} 
\end{align}
For $t = 0$, $(\e_0, \F_0, \w_0)$ are given as an initial condition of the algorithm.

We recognize that the iteration (\ref{eq-emtd0}) has the same form as standard TD, but the trace $\e_t$ is calculated differently, involving an ``emphasis'' weight $\M_t$ on the state $S_t$, which itself evolves along with the iterate $\F_t$, called the ``follow-on'' trace. If $\M_t$ is always set to $1$ regardless of $\F_t$ and $\i(\cdot)$, then the iteration (\ref{eq-emtd0}) reduces to the standard TD($\lambda)$ in the case where $\gamma$ and $\lambda$ are constants.

To explain at a high level what ETD($\lambda$) aims to achieve with the weighting scheme (\ref{eq-td3})-(\ref{eq-td1}), let us discuss the approximate Bellman equation it aims to solve. Associated with ETD($\lambda$) is a generalized Bellman equation of which $v_\pi$ is the unique solution \citep{Sut95}:\footnote{For the details of this Bellman equation, we refer the readers to the early work \citep{Sut95,SuB} and the recent work \citep{SuMW14}.}
\begin{equation} \label{eq-bellman}
  v = \rl + \PL \, v.
\end{equation}  
Here $\PL$ is an $N \times N$ substochastic matrix, and $\rl \in \re^N $ is a vector of expected total rewards attained by $\pi$ up to some random time depending on the functions $\gamma$ and $\lambda$, given by
\begin{equation} \label{eq-bellman-def}
  \PL  = I - (I - \P \Gm \Lm)^{-1} \, (I - \P \Gm), \qquad \rl  = (I - \P \Gm \Lm)^{-1} \, \r,
\end{equation}  
where $\Lm$ is a diagonal matrix with diagonal entries $\lambda(s), s \in \S$.
ETD($\lambda$) aims to solve a projected version of the Bellman equation (\ref{eq-bellman}) \citep[see][]{SuMW14}:
\begin{equation} \label{eq-proj-bellman}
       v = \Pi \big( \rl + \PL \, v \big),  \quad v \in E,  \qquad \Longleftrightarrow \qquad \bA \w + \bb = 0, \quad \w \in \rn.
\end{equation}       
In the above, $\Pi$ is the projection onto $E$ with respect to a weighted Euclidean norm or seminorm. The weights that define this norm also define the diagonal entries of a diagonal matrix $\bM$, and are given by
\begin{equation} \label{eq-m}
  diag(\bM)= \bi^{\top} (I - \PL)^{-1}, \qquad
\text{with} \  \ \ \bi \in \re^N, \, \bi(s) = d_{\pi^o}(s) \cdot  \i(s), \ s \in \S,
\end{equation}
where $d_{\pi^o}(s) > 0$ denotes the steady state probability of state $s$ for the behavior policy $\pi^o$, under Assumption~\ref{cond-bpolicy}(ii).
For the corresponding linear equation in the $\w$-space in Eq.~(\ref{eq-proj-bellman}), 
\begin{align}  \label{eq-Ab}
    \bA    =  - \Fe^\top  \bM \, (I - \PL) \,  \Fe , \qquad \quad
    \bb   =  \Fe^\top  \bM \,  \rl.
\end{align}  

Important for the convergence of ETD($\lambda$) is the negative definiteness of $\bA$. It can be shown that under Assumption~\ref{cond-bpolicy}, $\bA$ is negative definite whenever $\bA$ is nonsingular (Prop.~\ref{prp-ndef}, Appendix~\ref{appsec-ndef}).
\footnote{Prior to our work, \citet{SuMW14} proved the negative definiteness of $\bA$ for positive $\i(\cdot)$ under Assumption~\ref{cond-bpolicy}.} 
By comparison, if we set $\M_t=1$ regardless of $\F_t$ and $\i(\cdot)$, the weights that define the projection norm and $diag(\bM)$ would simply become $d_{\pi^o}$, the same as in the regular off-policy TD($\lambda$). If we set $\M_t=\i(s)$, then the weights are given by $\bi$. Neither of these cases guarantees $C$ to be negative definite, unless $\lambda$ is sufficiently close to $1$.
Having the desirable negative definiteness property of $C$ 
is one of the motivations for introducing the weighting scheme (\ref{eq-td3})-(\ref{eq-td1}) in ETD($\lambda$) \citep{SuMW14}. 

For the convergence analysis in this paper, we shall assume:

\begin{assumption}[Nonsingularity condition] \label{cond-A}
The matrix $\bA$ given in Eq.~(\ref{eq-Ab}) is nonsingular.
\end{assumption}

We remark that for ETD($\lambda$) under Assumption~\ref{cond-bpolicy}, $C$ is always negative semidefinite \citep{SuMW14} (cf.\ our Prop.~\ref{prp-ndef0}, Appendix~\ref{appsec-ndef}), and the nonsingularity condition above is indeed equivalent to $C$ being negative definite (Prop.~\ref{prp-ndef}). This condition is fairly mild and allows $\i(s)=0$ for some states $s$. Specifically, as we prove in Appendix~\ref{appsec-ndef} (see Prop.~\ref{prp-ndef}), Assumption~\ref{cond-A} is equivalent to a condition on the approximation subspace, which requires merely that the set of feature vectors of those states with positive emphasis weights contains $n$ linearly independent vectors (cf.\ Remark~\ref{rmk-seminorm2}). Moreover, this requirement can be fulfilled easily without knowledge of the model (see Cor.~\ref{cor-ndef}, Remark~\ref{rmk-seminorm2}). We also note that when $\bA$ is negative definite, the projection $\Pi$ in Eq.~(\ref{eq-proj-bellman}) is well-defined (with respect to a seminorm if in Eq.~(\ref{eq-m}) some diagonal entries of $\bM$ equal zero), the projected Bellman equation (\ref{eq-proj-bellman}) has a unique solution, and bounds on the approximation error of ETD($\lambda$) can be derived using the approach of \citet{bruno-oblproj}. (For details of this discussion, see Remark~\ref{rmk-seminorm1} in Appendix~\ref{appsec-ndef}.) 

The ELSTD($\lambda$) algorithm aims to solve the same projected Bellman equation (\ref{eq-proj-bellman}) as ETD($\lambda$).
ELSTD($\lambda$) 
calculates iteratively an $n \times n$ matrix $\bA_t$ and a vector $\bb_t \in \rn$ according to
\begin{align}
  \bA_{t+1}  & = ( 1 - \alpha_t) \, \bA_t + \alpha_t \, \e_t \cdot \rho_t \big( \gamma_{t+1} \fe(S_{t+1})^\top - \fe(S_t)^\top \big),  \label{eq-lstd1} \\
  \bb_{t+1}  & =  ( 1 - \alpha_t) \, \bb_t + \alpha_t \,  \e_t \cdot \rho_t \, R_{t}, \label{eq-lstd2} 
\end{align}  
where the trace $\e_t$ is calculated according to Eqs.~(\ref{eq-td3})-(\ref{eq-td1}) as in ETD($\lambda$).
ELSTD$(\lambda)$ sets $\w_t = - \bA_{t}^{-1} \bb_t$, the solution to $\bA_{t} \w + \bb_t = 0$, when $\bA_t$ is invertible.

Like ETD($\lambda$), without the weighting scheme~(\ref{eq-td3})-(\ref{eq-td1}), ELSTD($\lambda$) would reduce essentially to the regular LSTD($\lambda$) (see e.g., \citep{lstd,Yu-siam-lstd} for on-policy and off-policy LSTD($\lambda$)).

\subsection{Convergence Results} 
We analyze ETD($\lambda$) and ELSTD($\lambda$) with diminishing stepsizes. Summarized below are their convergence properties,
which we will establish in the rest of this paper. 
In what follows, we denote by $\| \cdot\|$ the infinity norm for both vectors and matrices (viewed as vectors). For different stepsize conditions, our results will involve different convergence modes: convergence in $L^1$,
\footnote{For vector-valued random variables $X$, $X_t, t \geq 0$, 
by ``$\{X_t\}$ converges to $X$ in $L^1$'' we mean $\E [ \| X_t - X\| ] \overset{t \to \infty}{\to} 0$.}
in probability, or almost sure (a.s.) convergence (we write $\asto$ for ``converges almost surely''). 
First, we state a general stepsize condition that we will use.

\begin{assumption}[Stepsize condition] \label{cond-stepsize}
The stepsize sequence $\{ \alpha_t\}$ is deterministic and eventually nonincreasing, and satisfies $\alpha_t \in (0,1]$, $\sum_t \alpha_t = \infty$, $\sum_t \alpha_t^2 < \infty$.
\end{assumption}

Under the above condition we may take $\alpha_t = t^{-c}, c \in (1/2,1]$. However, stepsizes decreasing as $t^{-1}$ will be required in our almost sure convergence results; some cases will require $\alpha_t = O(1/t)$ with $\tfrac{\alpha_t - \alpha_{t+1}}{\alpha_t}= O(1/t)$.
\footnote{We write $\delta_t = O(1/t)$ for a scalar sequence $\{\delta_t\}$, if for some $c > 0$, $0 \leq \delta_t \leq c/t$ for all $t$.} 
(For instance, $\alpha_t = c_1/(c_2 +t)$ for some constants $c_1, c_2 > 0$.)

Our results are as follows. Let $\w^*$ denote the desired limit for ETD($\lambda$):
$$\w^* = - \bA^{-1} \bb, \qquad \text{for  $\bA$,   $\bb$  defined by Eq.~(\ref{eq-Ab}) under Assumptions~\ref{cond-bpolicy},~\ref{cond-A}.}  $$

\begin{thm}[$L^1$ and almost sure convergence of ELSTD($\lambda$) Iterates] 
\label{thm-lstd} \hfill \\
Under Assumptions~\ref{cond-bpolicy},~\ref{cond-stepsize}, for any given initial $(\e_0,\F_0, \bA_0,\bb_0)$, the sequence $\{(\bA_t, \bb_t)\}$ generated by the ELSTD($\lambda$) algorithm (\ref{eq-lstd1})-(\ref{eq-lstd2}) converges in $L^1$:
$$ \lim_{t \to \infty} \E \big[ \big\| \bA_t - \bA \big\| \big] = 0,   \qquad  \lim_{t \to \infty} \E \big[ \big\| \bb_t - \bb \big\| \big] = 0.$$
If in addition the stepsize is given by $\alpha_t = 1/(t+1)$, then 
$\bA_t \asto \bA$, $\bb_t \asto \bb.$
\end{thm} 

The preceding theorem yields immediately the convergence of the parameter sequence $\{\w_t\}$ generated by ELSTD($\lambda$):

\begin{cor}[Convergence of ELSTD($\lambda$)]
Let Assumptions~\ref{cond-bpolicy}-\ref{cond-stepsize} hold. Let $\{\w_t\}$ be generated by the ELSTD($\lambda$) algorithm (\ref{eq-lstd1})-(\ref{eq-lstd2}) as $\w_t=-\bA_t^{-1} \bb_t$. Then for any given initial $(\e_0,\F_0, \bA_0,\bb_0)$, $\{\w_t\}$ converges to $\w^*$ in probability; if in addition $\alpha_t=1/(t+1)$, then $\w_t \asto \w^*$.
\end{cor} 

\begin{thm}[Almost sure convergence of ETD($\lambda$)] 
\label{thm-td-unconstrained}
Let Assumptions~\ref{cond-bpolicy}-\ref{cond-stepsize} hold. Let $\{\w_t\}$ be generated by the ETD($\lambda$) algorithm (\ref{eq-emtd0})  with stepsizes satisfying $\alpha_t = O(1/t) $ and $\tfrac{\alpha_t - \alpha_{t+1}}{\alpha_t}= O(1/t)$.
Then for any given initial $(\e_0, \F_0, \w_0)$,  $\w_t \asto \w^*$.
\end{thm}

\begin{rem}[On stepsizes] \rm
We believe that the range of stepsizes for the a.s.\ convergence of ELSTD($\lambda$) can be enlarged. If additional conditions on the behavior policy are imposed to restrict the variances of the trace iterates, it should also be possible to enlarge the range of stepsizes for ETD($\lambda$). These topics, as well as the use of random stepsizes, are under active investigation.
\end{rem}

\begin{rem}[On variances] \label{rmk-var} \rm
The preceding convergence results hold under almost minimal conditions on the behavior policy (Assumption~\ref{cond-bpolicy}(ii)). However, unless we  restrict sufficiently the behavior policy (which is difficult to do without knowledge of the model, when $\gamma \not<1$), the variances of the trace iterates can grow unboundedly (cf.\ Remark~\ref{rmk-behavior-ite}), significantly affecting the speed of convergence. 
This is a main difficulty in off-policy methods in general. Further research is required to overcome it. For a recent work in this direction, see \citep{wis14}.\end{rem}

\section{Properties of Trace Iterates and Convergence Analysis of ELSTD($\lambda$)} \label{sec-elstd}

In this section we analyze the trace iterates and convergence properties of ELSTD($\lambda$) iterates. The analysis not only leads to Theorem~\ref{thm-lstd} on the convergence of ELSTD($\lambda$), but also prepares the stage for the subsequent ODE-based convergence proof for ETD($\lambda$), by ensuring that ``local averaging'' gives the desired ``mean dynamics,'' as will be seen in Section~\ref{sec-etd}. 

The structure of our analysis will be similar to that of \citep{Yu-siam-lstd} for regular off-policy LSTD($\lambda$), but the proofs at intermediate steps are new and more involved. We will explain the key proof arguments in this section, and give the proof details and related results in Appendix~\ref{appsec-a}.

\subsection{Properties of Trace Iterates} \label{sec-elstd1}

Let $Z_t = (S_t, A_t, \e_t, \F_t)$ for $t \geq 0$; they form a Markov chain on $\S \times \A \times \re^{n+1}$.
First, we observe several important properties of the trace iterates $\{(\e_t, \F_t)\}$ and the Markov chain $\{Z_t\}$, under Assumption~\ref{cond-bpolicy}:
\begin{enumerate}
\item[(i)] For any given initial $(\e_0, \F_0)$, $\sup_{t \geq 0} \E \big[ \big\| (\e_t, \F_t) \big\| \big] < \infty$. (See Prop.~\ref{prp-bdtrace}.)
\item[(ii)] Let $\{(\e_t, \F_t)\}$ and $\{({\hat \e}_t, {\hat \F}_t)\}$ be defined by the same recursion (\ref{eq-td3})-(\ref{eq-td1}), using the same state and action random variables, but with different initial conditions $(\e_0, \F_0) \not=(\hat \e_0, \hat{\F}_0)$. Then
$\F_t - \hat{\F}_t \asto 0$ and $ \e_t - \hat{\e}_t  \asto \0$ (the zero vector in $\rn$).
(See Prop.~\ref{prp-2}.)
\item[(iii)] We can approximate the traces $(\e_t, \F_t)$, which depend on the entire history of past states and actions, by similarly defined ``truncated traces'' $(\tilde{\e}_{t,K}, \tilde{F}_{t,K})$ which depend on the most recent $2K$ states and actions only [cf.~Eqs.~(\ref{eq-tF})-(\ref{eq-te})]. The expected approximation ``error'' can be bounded uniformly in $t$, by a constant $\C_K$ which decreases to $0$ as $K \to \infty$. (See Prop.~\ref{prp-3}.) 
\item[(iv)] $\{Z_t\}$ is a weak Feller Markov chain
\footnote{A Markov chain $\{X_t\}$ on a metric space is \emph{weak Feller} if $\E [ f(X_1) \mid X_0 = x]$ is continuous in $x$ for every bounded continuous function $f$ on the state space \cite[Prop.~6.1.1(i)]{MeT09}. Using this and the fact that $(\e_1, \F_1)$ depends continuously on $(\e_0, \F_0)$ [cf.\ Eqs.~(\ref{eq-td3})-(\ref{eq-td1})], the weak Feller property of $\{Z_t\}$ can be seen.\label{footnote-weakFeller}}
and bounded in probability,
\footnote{A Markov chain $\{X_t\}$ on a topological space is \emph{bounded in probability} if, for each initial state $x$ and each $\epsilon > 0$, there exists a compact subset $D$ of the state space such that $\liminf_{t \to \infty} \mathbf{P}_x(X_t \in D) \geq 1 - \epsilon$, where $\mathbf{P}_x$ denotes the probability of events conditional on $X_0=x$ \cite[p.~142]{MeT09}. In our case, since $\S$ and $\A$ are finite, the property (i) above together with the Markov inequality implies that $\{Z_t\}$ is bounded in probability \citep[cf.][Lemma 3.4]{Yu-siam-lstd}.}
 and hence it has at least one invariant probability measure.
\footnote{By~\cite[Theorem 12.1.2(ii)]{MeT09}, a weak Feller Markov chain bounded in probability has at least one invariant probability measure. We mention that there is also an alternative, direct proof of the existence of an invariant probability measure for $\{Z_t\}$, which does not rely on the weak Feller property (see Appendix~\ref{appsec-invm-alt}).}
\end{enumerate}
Furthermore, as we will show in Theorem~\ref{thm-erg} below, $\{Z_t\}$ has a \emph{unique} invariant probability measure and is ergodic. 

These properties suggest that despite the growing variances, the trace iterates are well-behaved. Figure~\ref{fig:proof1} shows how the convergence results of this section, to be introduced next, will depend on these properties.

\begin{figure}[htbp] 
\centering
\includegraphics[width=0.95\textwidth]{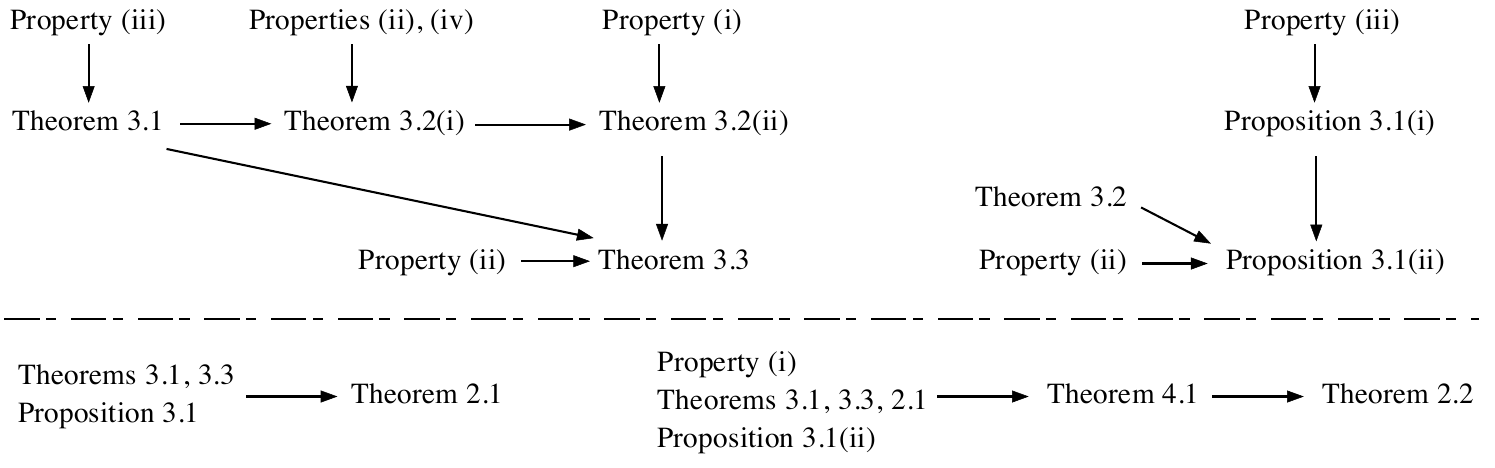}
\caption{Diagrams showing dependence relations between the results in this paper. ``$A \to B$'' means $A$ is used in proving $B$.}\label{fig:proof1}
\end{figure}

\subsection{Main Results on $L^1$ and Almost Sure Convergence} \label{sec-elstd2}

We formulate our convergence results in terms of a general recursion that can be specialized to the ELSTD($\lambda$) iteration. This generality is needed in order to make the results useful for other proofs, specifically, for proving the uniqueness of the invariant probability measure of $\{Z_t\}$, and for establishing convergence conditions required by an ODE-based analysis for ETD($\lambda$), as those proofs will rely on the convergence properties of certain iterates that are different from ELSTD($\lambda$).

We define the general recursion just mentioned as follows. Denote $y = (\e, \F)$; thus $y \in \re^{n+1}$. 
Consider a vector-valued function $h : \re^{n+1} \times \mathcal{S} \times \mathcal{A} \times \mathcal{S} \to \re^m$ such that $h(y, s, a, s')$ is Lipschitz continuous in $y$ for each $(s,a,s')$; i.e.,
there exists some constant $L_h$ such that for any $y, \hat y \in \re^{n+1}$,
\begin{equation} \label{cond-lip}
 \big\| h(y, s, a, s') - h(\hat y, s, a, s') \big\| \leq L_h  \| y - \hat y \|, \qquad \forall \,   (s, a, s') \in \mathcal{S} \times \mathcal{A} \times \mathcal{S}.
\end{equation} 
Given $h$, $\{Z_t\}$ and the stepsizes $\{\alpha_t\}$, we define a recursion as follows:
\begin{equation}  \label{eq-G}
  \G_{t+1} = (1 - \alpha_t) \, \G_t + \alpha_t \, h(Y_t, S_t, A_t, S_{t+1}).
\end{equation}  

The ELSTD($\lambda$) iterates $\bA_t$ and $\bb_t$ correspond to the following choices of $h$, respectively:
\begin{equation} \label{eq-h-lstd}
 h_1(y, s, a, s') =  \e \cdot \rho(s,a) \, \big( \gamma(s') \fe(s')^\top - \fe(s)^\top \big), \quad  h_2(y, s, a, s') =  \e \cdot \rho(s,a) \, r(s,a,s').
\end{equation} 
Here $h_1$ is matrix-valued (we view it as an $\re^m$-valued function with $m = n \times n$), and $h_2$ is $\re^n$-valued.
As just mentioned, we will also need to consider other choices of $h$ in our proofs later.

We first show that $\{G_t\}$ converges in $L^1$ to some constant vector. The proof (given in Appendix~\ref{appsec-prf-thml1}) exploits the property (iii) of truncated traces mentioned earlier: this property allows us to obtain the desired result by working with simple finite-state Markov chains. 

\begin{restatable}[$L^1$-convergence of $\{G_t\}$]{thm}{thmlone} 
\label{thm-l1conv}
Let $h$ be a vector-valued function satisfying the Lipschitz condition~(\ref{cond-lip}), and let $\{G_t\}$ be defined by the recursion~(\ref{eq-G}), using the process $\{Z_t\}$.
Then under Assumptions~\ref{cond-bpolicy},~\ref{cond-stepsize}, there exists a constant vector $G^*$ (independent of the stepsizes) such that for any given initial $Y_0 =(\e_0,\F_0)$ and $\G_0$,
$\lim_{t \to \infty} \E \big[ \big\| \G_t - \G^* \big\| \big] = 0.$
\end{restatable}

Next we analyze the a.s.\ convergence of $\{G_t\}$, by using ergodicity properties of the infinite-space Markov chain $\{Z_t\}$ that we establish first.
For each initial condition $\Z_0=z$, define the occupation probability measures $\mu_{z, t}$ for $t \geq 1$, by
$\mu_{z,t}(B) =  \frac{1}{t} \sum_{k=1}^{t} \mathbb{1}_B(Z_k)$ for any Borel subset $B$ of $\mathcal{S} \times \mathcal{A} \times \re^{n+1}$, 
where $\mathbb{1}_{B}$ denotes the indicator function for the set $B$ (i.e., $\mathbb{1}_{B}(x) = 1$ if $x \in B$, and $\mathbb{1}_{B}(x) = 0$ otherwise).
Let $\E_{\mu}$ denote expectation with respect to the probability distribution of the process $\{\Z_t\}$ with $\mu$ as the initial distribution of $\Z_0$.

\begin{thm}[Ergodicity of $\{\Z_t\}$] \label{thm-erg}
Under Assumption~\ref{cond-bpolicy}, the Markov chain $\{\Z_t\}$ has a unique invariant probability measure $\zeta$, and moreover, the following hold:\\
{\rm (i)} For each initial condition $Z_0=z$, the sequence $\{\mu_{z,t}\}$ of occupation measures converges weakly
\footnote{For probability measures $\mu, \mu_t, t \geq 0$, on a metric space $X$, $\{\mu_t\}$ is said to \emph{converge weakly} to $\mu$ if for all bounded continuous functions $f$ on $X$, $\int f d\mu_t \to \int f d\mu$ as $t \to \infty$ \cite[Chap.\ 9.3]{Dud02}.}
to $\zeta$, almost surely.\\
{\rm (ii)} $\E_{\zeta} \big[ \big\| h(Z_0, S_1) \big\| \big] < \infty$ for any function $h$ satisfying the Lipschitz condition~(\ref{cond-lip}).
\end{thm}

The preceding theorem follows from the properties of trace iterates given earlier and Theorem~\ref{thm-l1conv} (cf.\ Figure~\ref{fig:proof1}). The proof is the same as the corresponding proofs of \cite[Theorem 3.2 and Prop.\ 3.2]{Yu-siam-lstd} for the case of off-policy LSTD. In particular, to prove the uniqueness of the invariant probability measure (which is not as easy to prove as the existence given in the property (iv) earlier), we use the property (ii) and the convergence in $L^1$ result given in Theorem~\ref{thm-l1conv}.\footnote{Theorem~\ref{thm-l1conv} is useful here because on the separable metric space $\S \times \A \times \re^{n+1}$, bounded Lipschitz continuous functions are convergence-determining for weak convergence of probability measures \cite[Theorem 11.3.3]{Dud02}.}

We can now show that $\{G_t\}$ converges a.s.\ for stepsize $\alpha_t = 1/(t+1)$, by using the preceding results (cf.\ Figure~\ref{fig:proof1}), together with a strong law of large numbers for stationary processes \citep[Chap.\ X, Theorem 2.1]{Doob53} \citep[see also][Theorem 17.1.2]{MeT09}. The proof is a verbatim repetition of the proof of \cite[Theorem 3.3]{Yu-siam-lstd} and is therefore omitted.

\begin{thm}[Almost sure convergence of $\{G_t\}$] \label{thm-asconv}
Let $h$ and $\{\G_t\}$ be as in Theorem~\ref{thm-l1conv}, and let the stepsize be $\alpha_t = 1/(t+1)$. Then, under Assumption~\ref{cond-bpolicy}, for any given initial $Y_0=(\e_0,\F_0)$ and $\G_0$, $\G_t \asto \G^*$, where $\G^* = \E_\zeta \big[h(Y_0, S_0, A_0, S_1) \big]$ is the constant vector in Theorem~\ref{thm-l1conv}.
\end{thm}

Finally, we also need to analyze the cumulative effects of noise in the observed rewards $R_t$ and show that they diminish asymptotically. To this end, consider the following recursion: $W_0 = \0$ and
\begin{equation}  \label{eq-noiseW}
  W_{t+1} = (1 - \alpha_t) \, W_t + \alpha_t \, \e_t \, \rho_t \cdot \omega_{t+1}, \qquad t \geq 0,
\end{equation} 
where $\omega_{t+1} = R_{t} - r(S_t, A_t, S_{t+1})$ are noise variables.

\begin{restatable}[Effects of noise in random rewards]{prop}{propnoise} \label{prp-noise}
Under Assumptions \ref{cond-bpolicy},~\ref{cond-stepsize}, for any given initial $(\e_0, \F_0)$, we have 
{\rm (i)} $\E \big[  \| W_t \| \big] \to 0$; and 
{\rm (ii)} if, in addition, the stepsize is $\alpha_t = 1/(t+1)$,  then $W_t \asto \0$.
\end{restatable}

The proof of the preceding proposition is given in Appendix~\ref{appsec-prf-noise}. The proof of part (i) uses the property (iii) of truncated traces, similarly to the proof of Theorem~\ref{thm-l1conv}, and the proof of part (ii) is similar to that of Theorem~\ref{thm-asconv} (cf.\ Figure~\ref{fig:proof1}).

The convergence of ELSTD($\lambda$) stated in Theorem~\ref{thm-lstd} now follows from the preceding results (cf.\ Figure~\ref{fig:proof1}). Specifically, we calculate the limit $\G^*$ in Theorem~\ref{thm-l1conv} for the two functions $h_1, h_2$ in Eq.~(\ref{eq-h-lstd}), which are associated with the ELSTD($\lambda$) iterates $\{\bA_t\}, \{\bb_t\}$, respectively, and we show that $\G^* = \bA$ for $h = h_1$ and $\G^* = \bb$ for $h = h_2$. We also write the iterates $\{\bb_t\}$ equivalently as $\bb_{t+1}  = \G_{t+1} + W_{t+1}$ with $h = h_2$ in the definition of $\{G_t\}$. Then, the $L^1$-convergence part of Theorem~\ref{thm-lstd} follows from Theorem~\ref{thm-l1conv} and Prop.~\ref{prp-noise}(i), and the a.s.\ convergence part of Theorem~\ref{thm-lstd} follows from Theorem~\ref{thm-asconv} and Prop.~\ref{prp-noise}(ii). The complete proof with all the details is given in Appendix~\ref{appsec-prflstd}.

\section{Convergence Analysis of ETD($\lambda$)} \label{sec-etd}
Recall that ETD($\lambda$) calculates iteratively $\w_t$, $t \geq 0$, according to
\begin{equation} \label{eq-emtd}
  \w_{t+1} = \w_t + \alpha_t \, \e_t \cdot \rho_t \, \big( R_{t} + \gamma_{t+1} \fe(S_{t+1})^\top \w_t - \fe(S_t)^\top \w_t \big).
\end{equation}
Using the results of Section~\ref{sec-elstd}, we can now analyze its convergence by applying a ``mean ODE''
method from stochastic approximation theory \citep{KuY03}.

Denoting $\tilde \omega_{t+1} = \rho_t \, (R_{t} - r(S_t, A_t, S_{t+1}))$, let us write the iteration (\ref{eq-emtd}) equivalently as
\begin{equation} 
 \w_{t+1} = \w_t + \alpha_t \, h(\w_t, \xi_t) + \alpha_t \, \e_t \cdot \tilde \omega_{t+1},
\end{equation} 
where $\xi_t = (\e_t, S_t, A_t, S_{t+1})$ and 
$h : \rn \times  \rn \times \mathcal{S} \times \mathcal{A}  \times \mathcal{S} \to \rn$ is given by
\begin{equation} \label{eq-mfn1}
  h(\w, \xi) =  \e \cdot \rho(s, a) \, \big( r(s, a, s') + \gamma(s') \, \fe(s')^\top \w - \fe(s)^\top \w \big), \quad \text{for} \ \ \xi = (\e, s, a, s').
\end{equation}  
We will apply \cite[Theorem 6.1.1]{KuY03} to analyze the convergence of $\{\w_t\}$ generated by (\ref{eq-emtd}). 
The ``mean ODE'' associated with ETD($\lambda$) (\ref{eq-emtd}) is
\begin{equation} \label{eq-ode}
  \dot{x} = \bar h(x), \qquad \text{where} \ \  \bar h(x) = \bA x + \bb.
\end{equation}  
When $\bA$ is negative definite, the above ODE has a unique bounded (constant) solution $x(\cdot) \equiv \w^* = - \bA^{-1} \bb$ on the time interval $(-\infty, +\infty)$,  
and $\w^*$ is globally asymptotically stable for (\ref{eq-ode}) in the sense of Liapunov \citep[cf.][p.\ 23-24]{KuC78}. (A Liapunov function in this case is given by $\| \w - \w^* \|_2^2$, where $\| \cdot \|_2$ denotes the Euclidean norm.) 

However, the a.s.\ boundedness of $\{ \w_t\}$ is not easy to prove directly, which has prevented us from getting the desired convergence $\w_t \asto \w^*$ from \cite[Theorem 6.1.1]{KuY03} directly. For this reason, we analyze first a constrained version of (\ref{eq-emtd}) and establish its convergence. The result will then help the convergence analysis of the unconstrained algorithm (\ref{eq-emtd}) in Section~\ref{sec-td2}. 

\subsection{Convergence of Constrained ETD($\lambda$)} \label{sec-constrained-etd}

Consider the following constrained ETD($\lambda$) algorithm:
\begin{equation} \label{eq-emtd-const}
 \w_{t+1} = \Pi_{\H} \Big( \w_t + \alpha_t \, h(\w_t, \xi_t) + \alpha_t \, \e_t \cdot \tilde \omega_{t+1} \Big),
\end{equation} 
where $\H$ is a closed ball in $\rn$ with a sufficiently large radius $r$: $\H = \{ \w \in \re^n \mid \| \w \|_2 \leq r \}$, and $\Pi_\H$ is the Euclidean projection onto $\H$.
The ``mean ODE'' associated with the constrained algorithm~(\ref{eq-emtd-const}) is the projected ODE
\begin{equation} \label{eq-pode}
  \dot{x} = \bar h(x) + z, \qquad z \in - \mathcal{N}_\H(x),
\end{equation}  
where $\mathcal{N}_\H(x)$ is the normal cone of $\H$ at $x$, and $z$ is the boundary reflection term that cancels out the component of $\bar h(x)$ in $\mathcal{N}_\H(x)$ and is the ``minimal force'' needed to keep the solution in $\H$ \citep[Chap.\ 4.3]{KuY03}.
The negative definiteness of the matrix $\bA$ implies that the projected ODE (\ref{eq-pode}) has no stationary points other than $\w^*$ if the radius of $\H$ is sufficiently large:

\begin{restatable}{lem}{lmapode} \label{lma-pode}
Let $c > 0$ be such that $x^\top \bA x \leq - c \| x \|_2^2$ for all $x \in \rn$. Suppose $\H$ has a radius $r > \| \bb \|_2/c$. Then $\w^*$ lies in the interior of $\H$, and the only solution $x(t), t \in (-\infty, +\infty)$, of the projected ODE (\ref{eq-pode}) in $\H$ is $x(\cdot) \equiv \w^*$.
\end{restatable}

The proof of Lemma~\ref{lma-pode} is given in Appendix~\ref{appsec-etd}. We now apply \cite[Theorem 6.1.1]{KuY03} and Lemma~\ref{lma-pode} to prove the a.s.\ convergence of the constrained ETD($\lambda$) as stated in the theorem below. The proof is given in Appendix~\ref{appsec-etd}, and it uses the results of Section~\ref{sec-elstd} to verify the conditions required by \cite[Theorem 6.1.1]{KuY03}.

\begin{restatable}[Almost sure convergence of constrained ETD($\lambda$)]{thm}{thmctd} \label{thm-ctd}
Let Assumptions~\ref{cond-bpolicy}-\ref{cond-stepsize} hold.
Let $\{\w_t\}$ be the sequence generated by the constrained ETD($\lambda$) algorithm~(\ref{eq-emtd-const}) with stepsizes satisfying $\alpha_t = O(1/t)$ and $\tfrac{\alpha_t - \alpha_{t+1}}{\alpha_t}= O(1/t)$, and with the radius $r$ of $\H$ exceeding the threshold given in Lemma~\ref{lma-pode}.
Then, for any given initial $(\e_0,\F_0,\w_0)$, $\w_t \asto \w^*$. 
\end{restatable}

\subsection{Convergence of ETD($\lambda$)} \label{sec-td2}

We now prove the convergence theorem, Theorem~\ref{thm-td-unconstrained}, for the unconstrained ETD($\lambda$) algorithm by using the convergence of the constrained algorithm we just established. In particular, we shall compare the iterates generated by the unconstrained algorithm with those generated by the constrained one, and show that the difference between them diminishes asymptotically with probability one.

Let $\H = \big\{ \w \in \rn \mid \| \w \|_2 \leq r \big\}$ with its radius $r$ satisfying the condition of Lemma~\ref{lma-pode}. Note that to project $\w$ onto $\H$ is simply to scale $\w$: $\Pi_\H \w = \w$ if $\|\w\|_2 \leq r$; and $\Pi_\H \w = r \cdot \w /\| \w\|_2$ if $\|\w\|_2 > r$. More concisely, 
$$\Pi_\H \w = \eta \, \w, \qquad \text{where} \ \ \  \eta = \min \{ 1,  r / \| \w\|_2 \}.$$  
To simplify notation, define matrix $\mC_t$ and vector $g_{t}$ by
$$  \mC_t = \e_t \cdot \rho_t \,  \big( \gamma_{t+1} \, \fe(S_{t+1})  - \fe(S_t) \big)^\top, \qquad g_{t} = \e_t \cdot \rho_t \,  R_{t}.$$
Let us write the constrained algorithm (\ref{eq-emtd-const}) equivalently as
\begin{equation} \label{eq-prf-cemtd}
 \tilde{\w}_{t+1} = \left (I + \alpha_t \, \mC_t \right) \cdot \eta_t \, \tilde{\w}_{t} + \alpha_t \, g_{t},
\end{equation}  
where $\eta_0=1$ and $\eta_t = \min \{ 1,  r / \| \tilde{\w}_t\|_2 \}$ for $t \geq 1$. (For $t \geq 1$, $\eta_t \,\tilde{\w}_t$ corresponds to the projected iterate in (\ref{eq-emtd-const}), and $\tilde{\w}_t$ the iterate just before the projection.)
The unconstrained algorithm (\ref{eq-emtd}) can be equivalently written as
\begin{equation} \label{eq-prf-emtd}
  \w_{t+1} = \left (I + \alpha_t \, \mC_t \right) \cdot  \w_{t} + \alpha_t \, g_{t}.
\end{equation}

\begin{lem} \label{lma-pmtrx}
Under the conditions of Theorem~\ref{thm-ctd}, for any given initial $(\e_0, \F_0)$, almost surely, the sequence of matrices, $\prod_{k \geq \bar t}^t \left  (I + \alpha_k \, \mC_k \right)$, $t = \bar t, \bar t+1, \ldots$, converges to the $n \times n$ zero matrix as $t \to \infty$, for all $\bar t \geq 0$.
\end{lem}

\begin{proof}
It is sufficient to consider a given (arbitrary) vector $y \in \rn$ and prove that for each initial $(\e_0, \F_0)$ and each $\bar t \geq 0$, $\prod_{k \geq \bar t}^t \left  (I + \alpha_k \, \mC_k \right) y \asto \0$. To this end, consider generating the iterates $\tilde{\w}_{\bar t}, \tilde{\w}_{\bar t + 1}, \ldots,$ starting from time $\bar t$ and $\tilde{\w}_{\bar t} = y$, by using the constrained algorithm (\ref{eq-prf-cemtd}) as follows:  
$$ \tilde{\w}_{k+1} = \left (I + \alpha_k \, \mC_k \right) \cdot \eta_k \, \tilde{\w}_{k}, \qquad k \geq \bar t. $$
In the above, we calculate $(\e_k,\F_k)$ and $\mC_k$ as before starting from time $0$ and the given initial condition $(\e_0, \F_0)$, and we have set $g_k = R_k = 0$ for all $k$.
Notice that since the stepsize sequence $\{\alpha_t\}$ satisfies the condition of Theorem~\ref{thm-ctd}, so does
the stepsize sequence, $\alpha_{\bar t +1}, \alpha_{\bar t+2}, \ldots$. Then, in view of the Markovian property of $\{(S_t, A_t, \e_t, \F_t)\}$, we can apply Theorem~\ref{thm-ctd} to the above iteration starting from time $\bar t$ for each possible value of $(\e_{\bar t}, \F_{\bar t})$, thereby concluding that  
for the given $(\e_0,\F_0)$ and $\bar t$, $\tilde{\w}_t \asto \0$ (because $R_k = 0$ for all $k$ and the solution to $\bA \w =0$ is $\0$). 

On the other hand,
\begin{equation} \label{eq-prf-mtrx1}
  \tilde{\w}_{t+1} = \left( \textstyle{\prod}_{k \geq \bar t}^t \, \left(I + \alpha_k \, \mC_k \right) \right) \cdot \left( \textstyle{\prod}_{k \geq \bar t}^t  \, \, \eta_k \right) \cdot y.
\end{equation}  
Since the solution $\0$ lies in the interior of $\H$, if $\tilde{\w}_t \to \0$, then $\eta_k = 1$ for all $k$ sufficiently large. Thus the convergence $\tilde{\w}_t \asto \0$ implies 
that as $t \to \infty$, $\prod_{k \geq \bar t}^t  \eta_k$ converges a.s.\ to a strictly positive number that depends on the sample path and the vector $y$. Consequently, from Eq.~(\ref{eq-prf-mtrx1}) and the convergence $\tilde{\w}_t \asto \0$, we obtain that $\left( \prod_{k \geq \bar t}^t \left (I + \alpha_k \, \mC_k \right) \right) y \asto \0$ as $t \to \infty$. Now this holds for any given vector $y$, so by letting $y$ be each column of the identity matrix, it follows that as $t \to \infty$, the matrix $ \prod_{k \geq \bar t}^t \left (I + \alpha_k \, \mC_k \right)$ converges a.s.\ to the zero matrix.
\end{proof}

Finally, we prove the a.s.\ convergence of the unconstrained ETD($\lambda$) as stated by Theorem~\ref{thm-td-unconstrained}:
\smallskip
\begin{proofof}{Theorem~\ref{thm-td-unconstrained}} 
Let $\{\tilde{\w}_t\}$ be the iterates generated by the constrained algorithm (\ref{eq-prf-cemtd}) using the same trajectory of states, actions and rewards that are used by the unconstrained algorithm (\ref{eq-emtd}) to generate $\{\w_t\}$.
By Theorem~\ref{thm-ctd} and Lemma~\ref{lma-pmtrx}, there exists a set $\Omega_1$ of sample paths such that $\Omega_1$ has probability one and on $\Omega_1$,
$$ \tilde{\w}_t \to \w^* \qquad \text{and} \qquad  \lim_{t \to 0} \, \textstyle{\prod}_{k \geq \bar t}^t \left (I + \alpha_k \, \mC_k \right)  = 0_{n \times n}, \quad \forall \, \bar t \geq 0, $$
where $0_{n \times n}$ denotes the $n \times n$ zero matrix. Consider each path in $\Omega_1$. 
By our choice of the constraint set $\H$, $\w^*$ lies in the interior of $\H$ (Lemma~\ref{lma-pode}), so the convergence $\tilde{\w}_t \to \w^*$ implies the existence of a path-dependent time $t' < \infty$ such that $\eta_k=1$ for all $k \geq t'$. Then
$$\tilde{\w}_{k+1} = \left (I + \alpha_k \, \mC_k \right) \cdot \, \tilde{\w}_{k} + \alpha_k \, g_{k}, \qquad \forall \, k \geq t',$$
and consequently, 
\begin{align} 
 \w_{k +1} - \tilde{\w}_{k+1} & = \left (I + \alpha_k \, \mC_k \right) \cdot \big( \w_{k} - \tilde{\w}_{k} \big), \qquad \forall \, k \geq t', \notag \\
  \w_{t +1} - \tilde{\w}_{t+1} & = \left( \textstyle{\prod}_{k \geq t'}^t \left (I + \alpha_k \, \mC_k \right) \right) \cdot \big( \w_{t'} - \tilde{\w}_{t'} \big), \qquad \forall \, t \geq t'. \label{eq-prf-td}
\end{align}  
As $t \to \infty$, the matrix $\prod_{k \geq t'}^t \left (I + \alpha_k \, \mC_k \right)  \to 0_{n \times n}$ for the sample path under consideration. Thus, from Eq.~(\ref{eq-prf-td}) we obtain $\w_{t} - \tilde{\w}_{t} \to \0$; since $\tilde{\w}_t \to \w^*$, this implies $\w_t \to \w^*$.
\end{proofof}

\begin{rem}[Almost sure convergence of regular off-policy TD($\lambda$)] \rm
If $\lambda$ is a constant sufficiently close to $1$, the matrix associated with the ``mean updates'' of the regular off-policy TD($\lambda$) algorithm is also negative definite \citep{by08}. In that case, \cite[Prop.\ 4.1]{Yu-siam-lstd} established the a.s.\ convergence but only for a constrained version of the algorithm, similar to our Theorem~\ref{thm-ctd}. The proofs given in this subsection, combined with \cite[Prop.\ 4.1]{Yu-siam-lstd}, can be used to establish the desired a.s.\ convergence for the unconstrained off-policy TD($\lambda$) in that case.
\end{rem}

\acks{I thank Prof.\ Richard Sutton for discussions on the ETD algorithm, and Dr.\ Joseph Modayil, Ashique Mahmood, and several anonymous reviewers for helpful comments on the first version of this paper. This research was supported by a grant from Alberta Innovates -- Technology Futures.}

\addcontentsline{toc}{section}{References}
\bibliography{emTD_bib}

\begin{thebibliography}{34}
\providecommand{\natexlab}[1]{#1}
\providecommand{\url}[1]{\texttt{#1}}
\expandafter\ifx\csname urlstyle\endcsname\relax
  \providecommand{\doi}[1]{doi: #1}\else
  \providecommand{\doi}{doi: \begingroup \urlstyle{rm}\Url}\fi

\bibitem[Baird(1995)]{Baird95}
Baird, L.~C.
\newblock Residual algorithms: Reinforcement learning with function
  approximation.
\newblock In \emph{Proc. The 12th Int. Conf. Machine Learning}, pages 30--37,
  1995.

\bibitem[Bertsekas and Tsitsiklis(1996)]{BET}
Bertsekas, D.~P. and Tsitsiklis, J.~N.
\newblock \emph{Neuro-Dynamic Programming}.
\newblock Athena Scientific, Belmont, MA, 1996.

\bibitem[Bertsekas and Yu(2009)]{by08}
Bertsekas, D.~P. and Yu, H.
\newblock Projected equation methods for approximate solution of large linear
  systems.
\newblock \emph{Journal of Computational and Applied Mathematics}, 227\penalty0
  (1):\penalty0 27--50, 2009.

\bibitem[Billingsley(1968)]{Bil68}
Billingsley, P.
\newblock \emph{Convergence of Probability Measures}.
\newblock John Wiley \& Sons, New York, 1968.

\bibitem[Borkar(2008)]{Bor08}
Borkar, V.~S.
\newblock \emph{Stochastic Approximation: A Dynamic Viewpoint}.
\newblock Cambridge University Press, Cambridge, 2008.

\bibitem[Borkar and Meyn(2000)]{BorM00}
Borkar, V.~S. and Meyn, S.~P.
\newblock The {O.D.E.} method for convergence of stochastic approximation and
  reinforcement learning.
\newblock \emph{SIAM J. Control Optim.}, 38:\penalty0 447--469, 2000.

\bibitem[Boyan(1999)]{lstd}
Boyan, J.~A.
\newblock Least-squares temporal difference learning.
\newblock In \emph{Proc. The 16th Int. Conf. Machine Learning}, pages 49--56,
  1999.

\bibitem[Dann et~al.(2014)Dann, Neumann, and Peters]{dnp14}
Dann, C., Neumann, G., and Peters, J.
\newblock Policy evaluation with temporal differences: A survey and comparison.
\newblock \emph{Journal of Machine Learning Res.}, 15:\penalty0 809--883, 2014.

\bibitem[Doob(1953)]{Doob53}
Doob, J.~L.
\newblock \emph{Stochastic Processes}.
\newblock John Wiley \& Sons, New York, 1953.

\bibitem[Dudley(2002)]{Dud02}
Dudley, R.~M.
\newblock \emph{Real Analysis and Probability}.
\newblock Cambridge University Press, Cambridge, 2002.

\bibitem[Geist and Scherrer(2014)]{bruno14}
Geist, M. and Scherrer, B.
\newblock Off-policy learning with eligibility traces: A survey.
\newblock \emph{Journal of Machine Learning Res.}, 15:\penalty0 289--333, 2014.

\bibitem[Glynn and Iglehart(1989)]{gi-sampling}
Glynn, P.~W. and Iglehart, D.~L.
\newblock Importance sampling for stochastic simulations.
\newblock \emph{Management Science}, 35:\penalty0 1367--1392, 1989.

\bibitem[Kushner and Clark(1978)]{KuC78}
Kushner, H.~J. and Clark, D.~S.
\newblock \emph{Stochastic Approximation Methods for Constrained and
  Unconstrained Systems}.
\newblock Springer-Verlag, New York, 1978.

\bibitem[Kushner and Yin(2003)]{KuY03}
Kushner, H.~J. and Yin, G.~G.
\newblock \emph{Stochastic Approximation and Recursive Algorithms and
  Applications}.
\newblock Springer-Verlag, New York, 2nd edition, 2003.

\bibitem[Maei(2011)]{maei11}
Maei, H.~R.
\newblock \emph{Gradient Temporal-Difference Learning Algorithms}.
\newblock PhD thesis, University of Alberta, 2011.

\bibitem[Mahmood et~al.(2014)Mahmood, van Hasselt, and Sutton]{wis14}
Mahmood, A.~R., van Hasselt, H., and Sutton, R.~S.
\newblock Weighted importance sampling for off-policy learning with linear
  function approximation.
\newblock In \emph{Proc. Conf. Advances in Neural Information Processing
  Systems (NIPS) 27}, 2014.

\bibitem[Mahmood et~al.(2015)Mahmood, Yu, White, and Sutton]{MYWS15}
Mahmood, A.~R., Yu, H., White, M., and Sutton, R.~S.
\newblock Emphatic temporal-difference learning.
\newblock In \emph{European Workshops on Reinforcement Learning}, Lille,
  France, 2015.

\bibitem[Meyn(2008)]{Mey08}
Meyn, S.
\newblock \emph{Control Techniques for Complex Networks}.
\newblock Cambridge University Press, Cambridge, 2008.

\bibitem[Meyn and Tweedie(2009)]{MeT09}
Meyn, S. and Tweedie, R.~L.
\newblock \emph{Markov Chains and Stochastic Stability}.
\newblock Cambridge University Press, Cambridge, 2nd edition, 2009.

\bibitem[Neveu(1975)]{Nev75}
Neveu, J.
\newblock \emph{Discrete-Parameter Martingales}.
\newblock North-Holland, Amsterdam, 1975.

\bibitem[Precup et~al.(2001)Precup, Sutton, and Dasgupta]{offpolicytd-psd}
Precup, D., Sutton, R.~S., and Dasgupta, S.
\newblock Off-policy temporal-difference learning with function approximation.
\newblock In \emph{Proc. The 18th Int. Conf. Machine Learning}, pages 417--424,
  2001.

\bibitem[Puterman(1994)]{puterman94}
Puterman, M.~L.
\newblock \emph{Markov decision processes: {Discrete} stochastic dynamic
  programming}.
\newblock John Wiley \& Sons, 1994.

\bibitem[Randhawa and Juneja(2004)]{rj-tdimportance}
Randhawa, R.~S. and Juneja, S.
\newblock Combining importance sampling and temporal difference control
  variates to simulate {Markov} chains.
\newblock \emph{ACM Trans. Modeling and Computer Simulation}, 14\penalty0
  (1):\penalty0 1--30, 2004.

\bibitem[Rudin(1966)]{Rudin66}
Rudin, W.
\newblock \emph{Real and Complex Analysis}.
\newblock McGraw-Hill, New York, 1966.

\bibitem[Scherrer(2010)]{bruno-oblproj}
Scherrer, B.
\newblock Should one compute the temporal difference fix point or minimize the
  {Bellman} residual? {The} unified oblique projection view.
\newblock In \emph{Proc. The 27th Int. Conf. Machine Learning}, pages 959--966,
  2010.

\bibitem[Sutton(1988)]{Sut88}
Sutton, R.~S.
\newblock Learning to predict by the methods of temporal differences.
\newblock \emph{Machine Learning}, 3:\penalty0 9--44, 1988.

\bibitem[Sutton(1995)]{Sut95}
Sutton, R.~S.
\newblock {TD} models: Modeling the world at a mixture of time scales.
\newblock In \emph{Proc. The 12th Int. Conf. Machine Learning}, pages 531--539,
  1995.

\bibitem[Sutton and Barto(1998)]{SuB}
Sutton, R.~S. and Barto, A.~G.
\newblock \emph{Reinforcement Learning}.
\newblock MIT Press, Cambridge, MA, 1998.

\bibitem[Sutton et~al.(2015)Sutton, Mahmood, and White]{SuMW14}
Sutton, R.~S., Mahmood, A.~R., and White, M.
\newblock An emphatic approach to the problem of off-policy temporal-difference
  learning, 2015.
\newblock \url{http://arxiv.org/abs/1503.04269}.

\bibitem[Tsitsiklis(1994)]{tsi94}
Tsitsiklis, J.~N.
\newblock Asynchronous stochastic approximation and {Q}-learning.
\newblock \emph{Machine Learning}, 16:\penalty0 185--202, 1994.

\bibitem[Tsitsiklis and {Van Roy}(1997)]{tr-disc}
Tsitsiklis, J.~N. and {Van Roy}, B.
\newblock An analysis of temporal-difference learning with function
  approximation.
\newblock \emph{IEEE Trans. Automat. Contr.}, 42\penalty0 (5):\penalty0
  674--690, 1997.

\bibitem[Varga(2000)]{Var00}
Varga, R.~S.
\newblock \emph{Matrix Iterative Analysis}.
\newblock Springer-Verlag, Berlin, 2nd edition, 2000.

\bibitem[Yu(2012)]{Yu-siam-lstd}
Yu, H.
\newblock Least squares temporal difference methods: An analysis under general
  conditions.
\newblock \emph{SIAM J. Control Optim.}, 50:\penalty0 3310--3343, 2012.

\bibitem[Yu and Bertsekas(2012)]{yb-bellmaneq}
Yu, H. and Bertsekas, D.~P.
\newblock Weighted {Bellman} equations and their applications in approximate
  dynamic programming.
\newblock LIDS Technical Report 2876, MIT, 2012.

\end{thebibliography}

\clearpage
\addcontentsline{toc}{section}{Appendices}
\appendix

\section{Proof Details for Section~\ref{sec-elstd}} \label{appsec-a}
In this appendix we give proof details and related results for Section~\ref{sec-elstd}.
Assumption~\ref{cond-bpolicy} on the target and behavior policies will be in force throughout, so it will not be mentioned explicitly in intermediate technical results.

\subsection{Some Basic Technical Lemmas}

We prove three basic lemmas that will be useful later. 
First, recall that $\Gm$ and $\Lm$ are diagonal matrices with $\gamma(s)$ (discount factors) and $\lambda(s)$, $s \in \S$, on their diagonals, respectively. Note also that under Assumption~\ref{cond-bpolicy}, the inverse $(I - \P\Gm)^{-1}$ exists. This implies that $(I - \P\Gm\Lm)^{-1}$ also exists. 
Then, since 
$$(I - \P\Gm)^{-1} = \sum_{t =0}^\infty (\P \Gm)^{t}, \qquad  (I - \P\Gm\Lm)^{-1} = \sum_{t =0}^\infty (\P \Gm \Lm)^{t},$$
both $(\P \Gm)^{t}$ and $(\P \Gm \Lm)^{t}$ converge to the zero matrix as $t \to \infty$. 

We now specify some notation. In what follows, let $\1$ denote the vector of all ones. For an expression $H$ that results in a vector in $\re^N$, we will write $(H)(s)$ for the $s$-th entry of the resulting vector.
(For example, $(\P \1)(s)$ and $(\1^\top \P)(s)$ represent the $s$-th entry of the vector $\P \1$ and $\1^\top \P$, respectively.)

Let $\mathcal{F}_{t}=\sigma\big( S_0, A_0, \ldots, S_{t} \big)$ be the $\sigma$-algebra generated by the states and actions up to time $t$, including the state $S_t$ but excluding the action $A_t$. Recall some shorthand notation we defined earlier:
$$ \rho_t = \rho(S_t, A_t) = \tfrac{\pi( A_t \mid S_t)}{\pi^o(A_t \mid S_t)}, \qquad \gamma_t = \gamma(S_t), \qquad \lambda_t = \lambda(S_t).$$
To simplify notation, let us also define for $t \geq 1$,
$$  \beta_t = \rho_{t-1} \, \gamma_t \, \lambda_t.$$

\begin{lem} \label{lma1}
For all $t > k \geq 0$, 
\begin{align}
   \E \big[\, \rho_{k} \gamma_{k+1} \cdots \rho_{t-1} \gamma_t  \, \mid \mathcal{F}_k \big] & = \big( (\P \Gm)^{t-k} \1 \big)(S_k) \leq 1, \label{eq-lma1-1} \\
   \E \big[ \, \beta_{k+1} \beta_{k+2} \cdots \beta_t \, \mid \mathcal{F}_k \big] & = \big( (\P \Gm \Lm) ^{t-k} \1 \big)(S_k)  \leq 1. \label{eq-lma1-2}
\end{align}   
Furthermore, as $t \to \infty$, $\prod_{k=1}^t \big( \rho_{k-1} \gamma_{k} \big) \asto 0$ and $\prod_{k=1}^t \beta_k \asto 0$.
\end{lem}

\begin{proof}
The first two equations follow simply from a direct calculation. 

Let $\Delta_t = \prod_{k=1}^t \big( \rho_{k-1} \gamma_{k} \big)$. To prove $\Delta_t \asto 0$, 
consider equivalently the iterates 
$$\Delta_0=1, \qquad \Delta_t = (\rho_{t-1} \gamma_t) \Delta_{t-1}, \ \ \  t \geq 1.$$ 
Clearly $\Delta_t$ is $\mathcal{F}_t$-measurable, and by Eq.~(\ref{eq-lma1-1}) with $k=t-1$, 
$$\E \big[ \Delta_t \mid \mathcal{F}_{t-1} \big] = \Delta_{t-1} \cdot \E \big[ \rho_{t-1} \gamma_t \mid \mathcal{F}_{t-1}  \big] \leq \Delta_{t-1}.$$ 
So $\{(\Delta_t, \mathcal{F}_t)\}$ is a nonnegative supermartingale with $\E [\Delta_0] = 1 < \infty$. By a convergence theorem for nonnegative supermartingales \cite[Theorem II-2-9]{Nev75}, $\Delta_t \asto \Delta_\infty$ for some nonnegative random variable $\Delta_\infty$ satisfying $\E \big[ \Delta_\infty \big] \leq \liminf_{t \to \infty} \E \big[ \Delta_t \big]$. From Eq.~(\ref{eq-lma1-1}) with $k=0$, we have 
$\E \big[ \Delta_t \big] \leq \1^\top (\P \Gm)^t \1 \to 0$ as $t \to \infty$; therefore, $\E \big[ \Delta_\infty\big] = 0$. This implies $\Delta_\infty = 0$ a.s., i.e., $\Delta_t \asto 0$.

The assertion $\prod_{k=1}^t \beta_k \asto 0$ follows similarly by considering the iterates $\Delta_t = \beta_t \Delta_{t-1}$ with $\Delta_0 = 1$, and by using Eq.~(\ref{eq-lma1-2}) together with the nonnegative supermartingale convergence argument.
\end{proof}

\begin{lem} \label{lma2}
For $k \geq 0$, let $Y_k$ be an $\mathcal{F}_k$-measurable nonnegative random variable. Then for $t > k$,
\begin{align}
  \E \big[ Y_k \cdot (\rho_{k} \gamma_{k+1} \cdots \rho_{t-1} \gamma_t \big) \big] & \leq \E [ Y_k ] \cdot \big( \1^\top (\P \Gm) ^{t-k} \1 \big), \label{eq-lma-a1} \\
 \E \big[ Y_k \cdot \big(\beta_{k+1} \beta_{k+2} \cdots \beta_t \big) \big] & \leq \E [ Y_k ] \cdot \big( \1^\top (\P \Gm \Lm) ^{t-k}  \1 \big). \label{eq-lma-b1}
\end{align} 
Hence, if for some constant $\C < \infty$, $\E [ Y_k] \leq \C$ for all $k$, then
\begin{align}
 \E \left[ \sum_{k=0}^t Y_k \cdot (\rho_{k} \gamma_{k+1} \cdots \rho_{t-1} \gamma_t \big) \right] & \leq \C \cdot \1^\top  \left( \sum_{k=0}^t (\P \Gm) ^{k}  \right) \1 < \infty, \label{eq-lma-a2}\\
 \E \left[ \sum_{k=0}^t Y_k \cdot \big(\beta_{k+1} \beta_{k+2} \cdots \beta_t \big) \right] & \leq \C \cdot  \1^\top \left( \sum_{k=0}^t (\P \Gm \Lm) ^{k}  \right)  \1 < \infty. \label{eq-lma-b2}
\end{align} 
\end{lem}

\begin{proof}
For any $k \geq 0$, $ \big( (\P \Gm) ^{t-k} \1 \big)(S_k) \leq \1^\top (\P \Gm) ^{t-k} \1$. Using this, Eq.~(\ref{eq-lma1-1}) in Lemma~\ref{lma1}, and the  
assumption that $Y_k$ is $\mathcal{F}_k$-measurable, we have
\begin{align*}
 \E \left[ Y_k \cdot \big( \rho_k \gamma_{k+1} \cdots \rho_{t-1} \gamma_t \big) \right] & =  \E \big[ Y_k \cdot \E \left[ (\rho_{k} \gamma_{k+1} \cdots \rho_{t-1} \gamma_t \big) \mid \mathcal{F}_k \right] \big]  \leq \E [ Y_k] \cdot \big( \1^\top (\P \Gm) ^{t-k} \1 \big).
 \end{align*}
This proves Eqs.~(\ref{eq-lma-a1}), (\ref{eq-lma-a2}). Similarly, Eqs.~(\ref{eq-lma-b1}), (\ref{eq-lma-b2}) are obtained by using Eq.~(\ref{eq-lma1-2}) in Lemma~\ref{lma1} and a direct calculation.
\end{proof}

\begin{lem}  \label{lma3}
Let $\{a_k\}$ and $\{c_k\}$ be two sequences of nonnegative numbers with $\sum_{k=1}^\infty a_k < \infty$ and $\sum_{k=1}^\infty c_k < \infty$. 
Then $\lim_{t \to \infty} \sum_{k=1}^t a_k \, c_{t-k} = 0$.
\end{lem}

\begin{proof}
For any $m \leq t$,
$$ \sum_{k=1}^t a_k \, c_{t-k}  = \sum_{k=1}^m a_k \, c_{t-k} + \sum_{k=m+1}^t a_k \, c_{t-k} 
\leq  \big( \max_k a_k \big) \cdot \sum_{k=t-m}^{t-1} c_k + (\max_k c_k) \cdot \sum_{k=m+1}^t a_k.$$
Since $\{a_k\}$ and $\{c_k\}$ are summable by assumption, we have $\max_k a_k < \infty$, $\max_k c_k < \infty$, and $\lim_{t \to \infty} \sum_{k=t-m}^{t-1} c_k = 0$. So if we fix $m$ and let $t$ go to infinity in the preceding inequality, we have
$$ \limsup_{t \to \infty}  \, \sum_{k=1}^t a_k \, c_{t-k} \leq (\max_k c_k) \cdot \sum_{k=m+1}^\infty a_k.$$
Since $\lim_{m \to \infty} \sum_{k=m+1}^\infty a_k = 0$ by the summable assumption, 
by letting $m$ go to infinity in the right-hand side above, we obtain $\lim_{t \to \infty} \sum_{k=1}^t a_k \, c_{t-k}  = 0$.
\end{proof}

\subsection{Properties of the Trace Iterates $\{(\e_t, \F_t)\}$}
In this subsection we state formally the properties of trace iterates which we mentioned in Section~\ref{sec-elstd1}, and we give their proofs. These properties will be used frequently in obtaining some of our main convergence theorems.

The following proposition is the property (i) mentioned in Section~\ref{sec-elstd1}. 
First, let us express the traces $\e_t, \F_t$, by using their definitions [cf.\ Eqs.~(\ref{eq-td3})-(\ref{eq-td1})], as
\begin{align}
  \F_{t} & = \F_0 \cdot \big(\rho_{0} \gamma_{1} \cdots \rho_{t-1} \gamma_t \big) + \sum_{k=1}^t \i(S_k) \cdot \big(\rho_{k} \gamma_{k+1} \cdots \rho_{t-1} \gamma_t \big),  \label{eq-F} \\
  \e_t & = \e_0 \cdot \big(\beta_{1} \cdots \beta_t \big) + \sum_{k=1}^t \M_{k} \cdot \fe(S_k) \cdot \big(\beta_{k+1} \cdots \beta_t \big), \label{eq-e}
\end{align}
where $\beta_k = \rho_{k-1} \gamma_k \lambda_k$ as defined in the previous subsection, and
$$ \M_k =  \lambda_k \, \i(S_k) + ( 1 - \lambda_k ) \, \F_k.$$
Let $\mathcal{F}_{t}=\sigma\big( S_0, A_0, \ldots, S_{t} \big)$ for $t \geq 0$, throughout this subsection.

\begin{prop} \label{prp-bdtrace}
For any given initial $(\e_0, \F_0)$, $\sup_{t \geq 0} \E \big[ \big\| (\e_t, \F_t) \big\| \big] < \infty$.
\end{prop}

\begin{proof}
Let us calculate $\E \big[ | \F_t | \big]$ and $\E \big[ \| \e_t\| \big]$. Since the number of states is finite, there exists a finite constant $\C > 0$ such that $\C \geq \|(\e_0, \F_0)\|$ and $\C \geq \i(s)$, $\C \geq \| \fe(s)\|$ for all states $s$. 
Using the expression~(\ref{eq-F}) for $\F_t$ and applying Eq.~(\ref{eq-lma-a2}) in Lemma~\ref{lma2} (with $Y_0 = |F_0|$, $Y_k=\i(S_k)$, $k \geq 1$), we have the bound
\begin{align*}
  \E \big[ |\F_t | \big] 
  \leq \C \cdot \1^\top  \left( \sum_{k=0}^t (\P \Gm) ^{k}  \right) \1  \leq \C \cdot \1^\top ( I - \P \Gm) ^{-1} \1.
\end{align*}
Thus $\sup_{t \geq 0} \E \big[ |\F_t| \big] < \infty$. We now calculate $\E \big[ \| \e_t\|\big]$. Using the expression~(\ref{eq-e}) for $\e_t$, and using also the fact $\M_k \leq \C + |\F_k|$, we can bound $\| \e_t\|$ by
$$ \| \e_t\| \leq \C \cdot \big( \beta_1 \cdots \beta_t) + \C \cdot \sum_{k=1}^t \big( \C +  |\F_k | \big) \cdot \big( \beta_{k+1} \beta_{k+2} \cdots \beta_t \big).$$
Using the fact $\sup_{k \geq 0} \E \big[ |\F_k |\big] \leq \C'$ for some finite constant $\C'$ as we just proved, and using also Eq.~(\ref{eq-lma-b2}) in Lemma~\ref{lma2} (with $Y_0 = L$, $Y_k=L(L+|\F_k|), k \geq 1$), we obtain 
$$ \E \big[ \| \e_t\| \big] \leq \C (\C + \C' + 1) \cdot  \1^\top \left( \sum_{k=0}^t (\P \Gm \Lm) ^{k}  \right)  \1 \leq \C (\C + \C' + 1) \cdot  \1^\top ( I - \P \Gm \Lm) ^{-1} \1.$$
Hence $\sup_{t \geq 0} \E \big[\|\e_t\|\big] < \infty$. Since $\big\| (\e_t, \F_t) \big\| \leq \| \e_t \| + |\F_t|$, this shows that 
$$\sup_{t \geq 0} \E \big[ \big\| (\e_t, \F_t) \big\|  \big] \leq  \sup_{t \geq 0} \E \big[ \| \e_t\| \big] + \sup_{t \geq 0} \E \big[ |\F_t|\big] < \infty.$$
The proof is complete. 
\end{proof}

Recall that $\{Z_t\}$ with $Z_t = (S_t, A_t, \e_t, \F_t)$ denotes the Markov chain on the joint space $\S \times \A \times \re^{n+1}$ of states, actions and traces, and it is a weak Feller Markov chain (cf.\ Footnote~\ref{footnote-weakFeller}).
As explained in Section~\ref{sec-elstd1}, since $\S$ and $\A$ are finite, the preceding proposition implies that $\{Z_t\}$ is bounded in probability and hence, by its weak Feller property, has at least one invariant probability measure. We will need the following result (which is the property (ii) in Section~\ref{sec-elstd1}) to prove that $\{Z_t\}$ has a \emph{unique} invariant probability measure.

Let $({\hat \e}_t, {\hat \F}_t)$, $t \geq 1$, be defined by the same recursion (\ref{eq-td3})-(\ref{eq-td1}) that defines $(\e_t, \F_t)$, using the same state and action random variables, but with a different initial condition $(\hat \e_0, \hat{\F}_0)$. We write a zero vector in any Euclidean space as $\0$. 

\begin{prop} \label{prp-2}
For any two given initial conditions $(\e_0, \F_0)$ and $(\hat{\e}_0, \hat{\F}_0)$, 
$$ \F_t - \hat{\F}_t \asto 0, \qquad  \e_t - \hat{\e}_t  \asto \0.$$
\end{prop}

\begin{proof}
Using the expression (\ref{eq-F}) for $\F_t$ and $\hat{\F}_t$, we have 
$\F_t - \hat{\F}_t = ( \F_0 - \hat{\F}_0) \cdot \prod_{k=1}^t (\rho_{k-1} \gamma_k).$
Since $\prod_{k=1}^t (\rho_{k-1} \gamma_k) \asto 0$ by Lemma~\ref{lma1}, it follows that $\F_t - \hat{\F}_t \asto 0$. 

The proof of Lemma~\ref{lma1} also shows
\begin{equation} \label{eq-prf2a}
\E \big[ \big| \F_t - \hat{\F}_t \big| \big] \leq \big| \F_0 - \hat{\F}_0 \big| \cdot \1^\top (\P \Gm) ^t \1.
\end{equation}
We will need this inequality below.

We now prove $ \e_t - \hat{\e}_t \asto \0$. To simplify the derivation, we first observe that if $\hat\e_0\not=\e_0$ but $\hat{\F}_0=\F_0$, 
then $\hat\F_t = \F_t$ for all $t$, so the expression (\ref{eq-e}) for $\e_t$ and $\hat{\e}_t$ gives $\| \e_t - \hat{\e}_t\| =  \| \e_0 - \hat{\e}_0\| \cdot \prod_{k=1}^t \beta_k$. Since $\prod_{k=1}^t \beta_k \asto 0$ by Lemma~\ref{lma1}, it follows immediately that in this case $\| \e_t - \hat{\e}_t\| \asto 0$.

Thus, for the general case, we can focus on the difference between $\e_t$ and $\hat{\e}_t$ that is due to the difference between the initial $\F_0$ and $\hat{\F}_0$. In particular, define another sequence of iterates $(\tilde{\e}_t, \tilde {\F}_t)$ using the same recursion (\ref{eq-td3})-(\ref{eq-td1}) but with the initial condition $(\tilde{\e}_0, \tilde {\F}_0) = (\e_0, \hat{\F}_0)$. 
Since $\big\| \e_t - \hat{\e}_t \big\| \leq  \big\| \e_t - \tilde{\e}_t \big\| + \big\| \tilde{\e}_t - \hat{\e}_t \big\|$ and $\big\| \tilde{\e}_t - \hat{\e}_t \big\| \asto 0$  by what we just proved, to show $ \e_t - \hat{\e}_t \asto \0$, it is sufficient to prove $\| \e_t - \tilde{\e}_t\| \asto 0$. 
 
Since $\tilde {\F}_0 = \hat{\F}_0$, the sequence $\{\tilde{\F}_t\}$ coincides with $\{\hat{F}_t\}$. Then by the definition of $\e_t$ and $\tilde{\e}_t$,
$$ \e_t - \tilde{\e}_t = \beta_t \, \big( \e_{t-1} - \tilde{\e}_{t-1} \big)  +   (1 - \lambda_t ) \, \big( \F_t - \hat{\F}_t \big) \cdot \fe(S_t).$$
Since $0 \leq \E \big[ \beta_t \mid \mathcal{F}_{t-1}  \big] \leq 1$ (Lemma~\ref{lma1}) and $0 \leq 1 - \lambda_t \leq 1$, it follows that
\begin{equation} \label{eq-prf1}
  \E \big[ \big\| \e_t - \tilde{\e}_t \big\|  \mid \mathcal{F}_{t-1} \big] \leq \big\| \e_{t-1} - \tilde{\e}_{t-1} \big\|  + Y_{t-1}, \quad 
  \text{where} \ \ \ Y_{t-1} =  \E \big[  \big| \F_t - \hat{\F}_t \big| \cdot \| \fe(S_t) \| \mid \mathcal{F}_{t-1} \big].
\end{equation}  
Let us show $\sum_{t=0}^\infty Y_t < \infty$ a.s. In view of Eq.~(\ref{eq-prf1}), this will then imply, by a convergence theorem in \cite[Ex.\ II-4, p.\ 33-34]{Nev75} for nonnegative random processes (which is a consequence of the nonnegative supermartingale convergence theorem), that 
$\big\| \e_t - \tilde{\e}_t \big\|$ converges a.s.\ to a finite limit.

To prove $\sum_{t=0}^\infty Y_t < \infty$ a.s., it is sufficient to show $ \E \big[ \sum_{t=0}^\infty Y_t \big] < \infty.$
Let $\C = \max_{s \in \mathcal{S}} \| \fe(s) \|$. 
By Eq.~(\ref{eq-prf2a}), for each $t$,
$$  \E \big[  Y_t \big] = \E \big[ \big| \F_{t+1} - \hat{\F}_{t+1} \big| \cdot \| \fe(S_{t+1}) \| \, \big] 
\leq \C \big| \F_0 - \hat{\F}_0 \big| \cdot \1^\top (\P \Gm) ^{t+1} \1.$$
Since $ \sum_{t=0}^\infty (\1^\top (\P \Gm) ^{t+1} \1) \leq \1^\top ( I - \P \Gm) ^{-1} \1 < \infty$,
we obtain
$  \E \big[ \sum_{t=0}^\infty Y_t \big] = \sum_{t=0}^\infty \E \big[  Y_t \big] < \infty$. Hence $\sum_{t=0}^\infty Y_t < \infty$ a.s., and
as discussed earlier, this implies that 
$$\big\| \e_t - \tilde{\e}_t \big\| \asto \Delta_\infty$$ 
for a nonnegative real-valued random variable $\Delta_\infty$.

What remains to be proved is $\Delta_\infty = 0$ a.s. By Fatou's lemma \cite[Theorem 4.3.3]{Dud02}, 
\begin{equation} \label{eq-prf2}
   \E [ \Delta_\infty ] \leq \liminf_{t \to \infty} \, \E \big[ \big\| \e_t - \tilde{\e}_t \big\| \big].
\end{equation}   
We show that the right-hand side equals $0$. By a direct calculation (using Eq.~(\ref{eq-e}) and the fact $\tilde{\e}_0 = \e_0$), we can write
$$ \e_t - \tilde{\e}_t = \sum_{k=1}^t \fe(S_k) \cdot (\F_k - \hat{\F}_k) \cdot (1 - \lambda_k)  \cdot \big(  \beta_{k+1} \cdots \beta_t \big).$$
For each $k \geq 1$, using Lemma~\ref{lma2} and Eq.~(\ref{eq-prf2a}), we have
\begin{align*}
\E \big[ \big|\F_k - \hat{\F}_k \big| \cdot (1 - \lambda_k) \cdot \big( \beta_{k+1} \cdots \beta_t \big) \big] & \leq \E \big[ \big| \F_k - \hat{\F}_k \big| \big] \cdot \big(\1^\top  (\P \Gm \Lm)^{t-k}  \1 \big) \\
& \leq \big| \F_0 - \hat{\F}_0 \big| \cdot \big(\1^\top (\P \Gm) ^k \1 \big) \cdot \big(\1^\top (\P \Gm \Lm) ^{t-k}  \1 \big).
\end{align*}
From the preceding two relations, it follows that
\begin{equation}  \label{eq-prf2b}
\E \big[ \big\| \e_t - \hat{\e}_t \big\| \big] \leq \C \big| \F_0 - \hat{\F}_0 \big| \cdot \sum_{k=1}^t  \big(\1^\top (\P \Gm) ^k \1 \big) \cdot \big(\1^\top (\P \Gm \Lm) ^{t-k} \P  \1 \big).
\end{equation}
From Lemma~\ref{lma3} with $\{a_k\}$ and $\{c_k\}$ defined as $a_k = \1^\top (\P \Gm) ^k \1$ and $c_k =\1^\top (\P \Gm \Lm) ^{k} \1$ for $k \geq 1$, we have
$$ \lim_{t \to \infty}  \, \sum_{k=1}^t  \big(\1^\top (\P \Gm) ^k \1 \big) \cdot \big(\1^\top (\P \Gm \Lm) ^{t-k}  \1 \big) =  0.$$ 
Combining this with Eq.~(\ref{eq-prf2b}) gives $\liminf_{t \to \infty} \, \E \big[ \big\| \e_t - \tilde{\e}_t \big\| \big] = 0$, and consequently, $\E [ \Delta_\infty ] = 0$ by Eq.~(\ref{eq-prf2}). This implies $\Delta_\infty = 0$ a.s., i.e., $\big\| \e_t - \tilde{\e}_t \big\| \asto 0$.
\end{proof}

The next proposition is the property (iii) mentioned in Section~\ref{sec-elstd1}, which concerns approximating the trace iterates $(\e_t, \F_t)$ by truncated traces that depend on a fixed number of the most recent states and actions only. We will use this proposition subsequently to prove Theorem~\ref{thm-l1conv}: it allows us to work with simple finite-space Markov chains, instead of working with the infinite-space Markov chain $\{Z_t\}$ directly. 

For each integer $K \geq 1$,
we define the truncated traces $Y_{t,K}=(\tilde{\e}_{t,K}, \tilde{\F}_{t,K})$ as follows: 
$$Y_{t,K} = (\e_t, \F_t) \quad  \text{for} \ \  t \leq K,$$
and for $t \geq K+1$,
\begin{align}
    \tilde{\F}_{t,K} & = \sum_{k=t-K}^t \i(S_k) \cdot \big(\rho_{k} \gamma_{k+1} \cdots \rho_{t-1} \gamma_t \big), \label{eq-tF} \\
    \tilde{\M}_{t,K} & = \, \lambda_t \, \i(S_t) + ( 1 - \lambda_t) \tilde{\F}_{t,K}, \label{eq-tM} \\
    \tilde{\e}_{t,K} & = \sum_{k=t-K}^t \tilde{\M}_{k,K} \cdot \fe(S_k) \cdot \big(\beta_{k+1} \cdots \beta_t \big). \label{eq-te}
\end{align}

Denote the original traces by $Y_t = (\e_t, \F_t)$; recall that they can be expressed as in Eqs.~(\ref{eq-F})-(\ref{eq-e}).
We have the following result, in which the notation ``$\C_K \downarrow 0$'' means that $\C_K$ decreases monotonically to $0$ as $K \to \infty$, and in which $\Z_0=(S_0,A_0,\e_0,\F_0)$ as we recall:

{\samepage
\begin{prop} \label{prp-3} \hfill
\begin{enumerate}
\item[{\rm (i)}] For any given initial $Y_0=(\e_0, \F_0)$, there exist constants $\C_K, K \geq 1$, with $\C_K \downarrow 0$, such that 
$$ \E \left[ \big\| Y_{t} - Y_{t,K}  \big\| \right] \leq \C_K, \qquad \forall \, t \geq 0.$$
\item[{\rm (ii)}] There exist constants $\C_K, K \geq 1$, independent of the given initial value of $\Z_0$, such that $\C_K \downarrow 0$ and
$$ \E \left[ \big\| Y_{t,K'} - Y_{t, K} \big\| \right]  \leq \C_K, \qquad \forall \, K' \geq K, \ t > 2 K'.$$
\end{enumerate}
\end{prop} 
}

\begin{proof}
Let $\C = \max \{ \F_0, \, \max_{s \in \mathcal{S}} \i(s) \}$.
We first calculate $\F_t - \tilde{\F}_{t,K}$. By definition $\F_t - \tilde{\F}_{t,K} = 0$ for $t \leq K$. For $t \geq K+1$, using the expressions (\ref{eq-F}), (\ref{eq-tF}) of $\F_t$ and $\tilde{\F}_{t,K}$, we have
$$ \F_t - \tilde{\F}_{t,K} = \F_0 \cdot \big(\rho_{0} \gamma_{1} \cdots \rho_{t-1} \gamma_t \big) + \sum_{k=1}^{t-K-1} \i(S_{k}) \cdot \big(\rho_{k} \gamma_{k+1} \cdots \rho_{t-1} \gamma_t \big),$$
from which it follows by applying Eq.~(\ref{eq-lma-a1}) in Lemma~\ref{lma2} that
\begin{equation}
  \E \left[ \big|  \F_t - \tilde{\F}_{t,K}  \big| \right] \leq \C \cdot \1^\top \left(\sum_{k=K+1}^{t}  (\P \Gm) ^{k} \right) \1 \leq \C \cdot \1^\top \left(\sum_{k=K+1}^{\infty}  (\P \Gm) ^{k} \right) \1 \overset{def}{=} \C_K^{(1)}. \label{eq-prf3}
\end{equation}  

Similarly we bound $\e_t -  \tilde{\e}_{t,K}$. By definition $\e_t = \tilde{\e}_{t,K}$ for $t \leq K$. For $t \geq K+1$, using the expressions (\ref{eq-e}), (\ref{eq-te}) of $\e_t$ and $\tilde{\e}_{t,K}$, and using also the expressions (\ref{eq-td2}), (\ref{eq-tM}) of $\M_t$ and $\tilde{M}_{t,K}$, we have
\begin{align*}
  \e_t - \tilde{\e}_{t,K} & = \e_0 \cdot \big(\beta_{1} \cdots \beta_t \big) + \sum_{k=1}^{t-K-1} M_k \cdot \fe(S_k) \cdot \big( \beta_{k+1} \cdots \beta_t \big)  \\
  & \ \ \ \, + \sum_{k=t-K}^t \fe(S_k) \cdot \big( \F_k - \tilde{\F}_{k,K} \big) \cdot (1 - \lambda_k) \cdot \big(  \beta_{k+1} \cdots \beta_t \big).
\end{align*}
By Prop.~\ref{prp-bdtrace}, $\sup_{k \geq 0} \E [ \M_k ] < \infty$, so we can find a constant $\C' < \infty$ that is greater than $ \| \e_0\|$, $\max_{s \in \S} \| \fe(s) \|$  and $\big(\max_{s \in \S} \| \fe(s) \| \big) \cdot  \sup_{k \geq 0} \E [ \M_k ]$.
Then applying Eq.~(\ref{eq-lma-b1}) in Lemma~\ref{lma2}, and using also Eq.~(\ref{eq-prf3}), we obtain
\begin{align*}
   \E \left[ \big\|  \e_t - \tilde{\e}_{t,K} \big\| \right] & \leq \C' \cdot \1^\top \left( \sum_{k=0}^{t-K-1} (\P \Gm \Lm) ^{t-k}  \right)  \1 + \C' \cdot \sum_{k=t-K}^t \E \big[ \big| \F_k - \tilde{\F}_{k,K} \big| \big] \cdot \big( \1^\top (\P \Gm \Lm) ^{t-k} \1 \big) \\
   & \leq \C' \cdot  \1^\top \left( \sum_{k=K+1}^{t} (\P \Gm \Lm) ^{k}  \right)  \1  + \C' \cdot \C_K^{(1)} \cdot \left( \1^\top \sum_{k=0}^{K} (\P \Gm \Lm) ^{k} \1 \right) \\
   & \leq \C' \cdot  \1^\top \left( \sum_{k=K+1}^{\infty} (\P \Gm \Lm) ^{k}  \right)  \1  + \C' \cdot \C_K^{(1)} \cdot \left( \1^\top ( I - \P \Gm \Lm) ^{-1} \1 \right) \overset{def}{=} \C_K^{(2)}. 
\end{align*}

Now let $\C_K = \C_K^{(1)} + \C_K^{(2)}$. From the expressions of $\C_K^{(1)}$ and $\C_K^{(2)}$ given above, clearly, $\C_K \downarrow 0$ as $K \to \infty$.
Then, in view of the relation $\big\| Y_t - Y_{t,K} \big\| \leq  \big|  \F_t - \tilde{\F}_{t,K}  \big| + \big\|  \e_t - \tilde{\e}_{t,K} \big\|$, we obtain the desired bound
$ \E \left[ \big\| Y_{t} - Y_{t,K}  \big\| \right]  \leq  \E \left[ \big|  \F_t - \tilde{\F}_{t,K}  \big| \right] +  \E \left[ \big\|  \e_t - \tilde{\e}_{t,K} \big\| \right]  \leq \C_K$. 
This proves part (i).

Part (ii) is proved similarly. By the definition of the truncated traces, for $t > 2 K' \geq 2 K$, 
$$      \tilde{\F}_{t,K'} - \tilde{\F}_{t,K}  = \sum_{k=t-K'}^{t-K-1} \i(S_{k}) \cdot \big(\rho_{k} \gamma_{k+1} \cdots \rho_{t-1} \gamma_t \big),$$
and
\begin{align*}
   \tilde{\e}_{t,K'} - \tilde{\e}_{t,K} & = \sum_{k=t-K'}^{t-K-1} \tilde{M}_{k, K'} \cdot \fe(S_k) \cdot \big( \beta_{k+1} \cdots \beta_t \big)  \\
    & \quad + \sum_{k=t-K}^t \fe(S_k) \cdot \big( \tilde{\F}_{k,K'} - \tilde{\F}_{k,K} \big) \cdot (1 - \lambda_k) \cdot \big(  \beta_{k+1} \cdots \beta_t \big).
\end{align*}    
We then apply the same calculation as in the proof of part (i). When $t > 2 K'$, the truncated traces do not depend on the initial condition $(\e_0,\F_0)$.
Since the state and action spaces are finite, we can set the constants $\C, \C'$ to be independent of the initial condition of $\Z_0$. Part (ii) then follows.
\end{proof}

\begin{rem}[On the behavior of trace iterates] \label{rmk-behavior-ite}\rm
From the properties of $\{(\e_t, \F_t)\}$ given above and the ergodicity of the Markov chain $\{(S_t, A_t, \e_t, \F_t)\}$ shown in Theorem~\ref{thm-erg}, we see that these trace iterates are well-behaved. On the other hand, like in regular off-policy algorithms, these iterates can be unbounded almost surely and their variances can grow to infinity with time. There are no contradictions here. To illustrate this point, let us consider a simple example with just $1$ state and $2$ actions, $\S=\{1\}, \A=\{a_1,a_2\}$, where all actions result in a self-transition at state $1$. Let $\pi(a_1\!\mid 1) =1$ for the target policy $\pi$, and let $\pi^o(a_1\!\mid 1)= q < 1$ for the behavior policy $\pi^o$. Let the discount factor be a constant $\gamma < 1$.  Then for all $t$, 
$$  \E [ \, \gamma_t^2 \rho_{t-1}^2 \mid \mathcal{F}_{t-1} \, ] = \gamma^2/q.$$
Suppose $\gamma^2/q > 1$. Then even with $\i(1)=0$, if $\F_0 > 0$, the definition $\F_t = \gamma_t \rho_{t-1} \F_{t-1}$ implies that
$$ \E [ \, \F_t^2 \, ] = \E \big[ \, \E [ \, \gamma_t^2 \rho_{t-1}^2 \mid \mathcal{F}_{t-1} \, ] \cdot \F_{t-1}^2 \big] = (\gamma^2/q )^t \cdot \F_0^2 \to \infty, $$
yet since $\i(1) = 0$, $\{\F_t\}$ is also a supermartingale converging to $0$ a.s.\ (cf.~the proof of Lemma~\ref{lma1}).
For the case $\i(1)>0$, again $\E [ \, \F_t^2 \, ] \to \infty$ if $\gamma^2/q > 1$, and by \cite[Prop.\ 3.1]{Yu-siam-lstd} the sequence $\{\F_t\}$ is almost surely unbounded if $\gamma/q > 1$, yet $\{\F_t\}$ is bounded in probability in the sense described by Prop.~\ref{prp-bdtrace}.

As mentioned earlier in Remark~\ref{rmk-var}, it can be desirable to restrict the behavior policy so that the variances of the trace iterates do not grow to infinity. In the simple example above, this can be easily arranged. In the general case, however, if the state-dependent discount factor $\gamma(\cdot)$ can take the value $1$ for some states, then without knowledge of the MDP model, to sufficiently restrict the behavior policy seems to be a difficult task.
\end{rem}

\subsection{Proof of Theorem~\ref{thm-l1conv}} \label{appsec-prf-thml1}

For convenience, we restate Theorem~\ref{thm-l1conv} here. Recall that the theorem concerns the recursion
$$  \G_{t+1} = (1 - \alpha_t) \, \G_t + \alpha_t \, h(Y_t, S_t, A_t, S_{t+1}),$$
where $Y_t = (\e_t, \F_t)$, and the function $h$ is Lipschitz continuous in $y$: for some constant $L_h$,
$$  \big\| h(y, s, a, s') - h(\hat y, s, a, s') \big\| \leq L_h  \| y - \hat y \|, \quad \forall \, y, \hat y \in \re^{n+1}, \ \forall \,   (s, a, s') \in \mathcal{S} \times \mathcal{A} \times \mathcal{S}.$$

\thmlone*

\begin{proof}
The proof proceeds in three steps:

\noindent {\bf (i)} For each $K \geq 1$, we consider the truncated traces $Y_{t,K} = (\tilde{\e}_{t,K}, \tilde{\F}_{t,K})$, $t=0, 1, \ldots$, defined by Eqs.~(\ref{eq-tF})-(\ref{eq-te}). Correspondingly, we define iterates $\tilde{G}_{0,K}=\G_0$ and
$$ \tilde{\G}_{t+1,K} = (1 - \alpha_t) \, \tilde{\G}_{t,K} + \alpha_t \, h(Y_{t,K}, S_t, A_t, S_{t+1}).$$
For each $t$, $Y_{t,K}$ is a function of $(S_{t- 2 K}, A_{t- 2 K}, \ldots, S_t)$, so $h(Y_{t,K}, S_t, A_t, S_{t+1})$ can be viewed as a function of $X_t = (S_{t- 2 K}, A_{t-2 K}, \ldots, S_{t+1})$, where $\{X_t\}$ is a finite state Markov chain with a single recurrent class by Assumption~\ref{cond-bpolicy}(ii). Then, with $\E_0$ denoting the expectation under the stationary distribution of the Markov chain $\{ (S_t, A_t)\}$, we have, by a result from stochastic approximation theory \cite[Chap.\ 6, Theorem 7 and Cor.\ 8]{Bor08}, that under Assumption~\ref{cond-stepsize} on the stepsizes,  
\begin{equation} \label{eq-prf4a}
    \tilde{\G}_{t,K} \asto \G^*_K, \qquad \text{where} \ \ \G^*_K = \E_0 \big[ \, h (Y_{k,K}, S_k, A_k, S_{k+1})  \,\big] \ \ \forall \, k > 2 K.
\end{equation} 
Clearly, the vector $\G^*_K$ does not depend on the initial condition $(Y_0, \G_0)$ and the stepsizes $\{\alpha_t\}$. Since for all $t$, $\| \tilde{\G}_{t,K} \| \leq \C$ for some constant $\C < \infty$, we also have by the bounded convergence theorem 
\begin{equation} \label{eq-prf4b}
     \lim_{t \to \infty} \E \big[ \big\| \tilde{\G}_{t,K} - \G^*_K \big\| \big] = 0.
\end{equation}

\noindent {\bf (ii)} We show that as $K \to \infty$, $G^*_K$ converges to some vector $G^*$. For any $K' > K$, using the Lipschitz property of $h$ and Prop.~\ref{prp-3}(ii), we have that for $k > 2 K'$,
\begin{align*}
    \big\|  \G^*_{K'}  - \G^*_K \big\| & =  \left\| \E_0 \big[ h \big(Y_{k,K'}, S_k, A_k, S_{k+1} \big)  - h \big(Y_{k,K}, S_k, A_k, S_{k+1} \big) \big]  \right\| \\
      & \leq L_h \, \E_0 \big[ \big\| Y_{k,K'} - Y_{k,K} \big\| \big] \leq L_h \, \C_K,
\end{align*}
where $\C_K$ is some constant with $\C_K \downarrow 0$ as $K \to \infty$. This shows that $\{G^*_K\}$ is a Cauchy sequence and hence converges to some $\G^*$.

\noindent {\bf (iii)} We establish the theorem by bounding the differences between $\G_t$ and $\tilde{\G}_{t,K}$ for an increasing $K$. For each $K$,
$$ \limsup_{ t \to \infty} \E \big[ \big\| \G_t - \G^* \big\| \big] \leq \limsup_{t \to \infty} \E \big[ \big\| \G_t - \tilde{\G}_{t,K} \big\| \big]  
+ \limsup_{t \to \infty} \E \big[ \big\|  \tilde{\G}_{t,K}  - \G_K^*\big\| \big]   + \big\|\G_K^* - \G^* \big\|.$$
In the right-hand side, the second term equals $0$ by Eq.~(\ref{eq-prf4b}), and the last term converges to $0$ as $K \to \infty$, as we just showed in step (ii). Consider now the first term. Since
$$ G_{t+1} - \tilde{\G}_{t+1,K}  = ( 1 - \alpha_t) \big(G_{t} - \tilde{\G}_{t,K} \big) + \alpha_t \big( h(Y_t, S_t, A_t, S_{t+1}) - h(Y_{t,K}, S_t, A_t, S_{t+1}) \big)$$
and $ \big\| h(Y_t, S_t, A_t, S_{t+1}) - h(Y_{t,K}, S_t, A_t, S_{t+1}) \big\| \leq L_h \| Y_t - Y_{t,K}\|$ by the Lipschitz property of $h$,   
we have
\begin{align}
   \E \big[ \big\| G_{t+1} - \tilde{\G}_{t+1,K} \big\| \big] & \leq  ( 1 - \alpha_t) \E \big[ \big\|  G_{t} - \tilde{\G}_{t,K} \big\| \big] + \alpha_t L_h \E \big[ \| Y_t - Y_{t,K}\| \big]  \notag \\
   & \leq ( 1 - \alpha_t) \E \big[ \big\|  G_{t} - \tilde{\G}_{t,K} \big\| \big] + \alpha_t L_h \C_K, \label{eq-prf5}
\end{align}
where the second inequality follows from Prop.~\ref{prp-3}(i), which gives the constants $\C_K, K \geq 1$, with $\C_K \downarrow 0$. For each $K$, in view of Assumption~\ref{cond-stepsize} on the stepsize, the inequality (\ref{eq-prf5}) implies that 
$$ \limsup_{t \to \infty} \E \big[ \big\|  G_{t} - \tilde{\G}_{t,K} \big\| \big]  \leq L_h \C_K.$$
Then, since $\C_K \downarrow 0$, letting $K$ go to infinity in the right-hand side of the preceding inequality, it follows that $\lim_{t \to \infty} \E \big[ \big\| \G_t - \G^* \big\| \big] = 0$.
\end{proof}

\subsection{Handling Noisy Rewards: Proof of Prop.~\ref{prp-noise}} \label{appsec-prf-noise}

For convenience, we restate Prop.~\ref{prp-noise} below. For each $t \geq 0$, let $\omega_{t+1} = R_{t} - r(S_t, A_t, S_{t+1})$, the noise in the observed reward $R_{t}$. We consider the recursion~(\ref{eq-noiseW}):
$W_0 = \0$ and
\begin{equation} \label{eq-prf-noiseW}
     W_{t+1} = (1 - \alpha_t) \, W_t + \alpha_t \, \e_t \, \rho_t \cdot \omega_{t+1}, \qquad t \geq 0.
\end{equation}     
Recall that it is assumed in our MDP model that $R_t$ has mean $r(S_t, A_t, S_{t+1})$ and bounded variance; specifically, 
let $\mathcal{F}_t = \sigma (S_0, A_0, \ldots, S_{t+1})$ in what follows, and we have that for some constant $\C < \infty$,
\begin{equation} \label{eq-cond-noise}
   \E \big[ \, \omega_{t+1} \mid \mathcal{F}_t \big] = 0, \qquad \E \big[ \, \omega_{t+1}^2 \mid \mathcal{F}_t \big] < \C.
\end{equation}   

\smallskip
\propnoise*

Because the proofs of part (i) and part (ii) use quite different arguments, we give them separately below.

 \smallskip
\begin{proofof}{Prop.~\ref{prp-noise}(i)}
To simplify notation, denote $\tilde \omega_{t+1} = \rho_t \, \omega_{t+1}$.
Similarly to the proof of Theorem~\ref{thm-l1conv}, we first consider for each $K \geq 1$, the truncated traces $\{\tilde{\e}_{t,K}, t \geq 0\} $ given by Eq.~(\ref{eq-te}), and  we replace $\{\e_t\}$ in the recursion (\ref{eq-prf-noiseW}) by $\{\tilde{\e}_{t,K}\}$ to define iterates $\wtld W_{0,K} = \0$ and
$$ \wtld W_{t+1, K} = (1 - \alpha_t) \, \wtld W_{t,K} + \alpha_t \, \tilde{\e}_{t,K} \cdot \tilde \omega_{t+1}, \qquad t \geq 0. $$
Since the number of states and actions is finite, we can bound $\|\tilde{\e}_{t,K}\|$ by some finite constant for all $t$. Then, using Eq.~(\ref{eq-cond-noise}), we have that for all $t \geq 0$,
$$ \E \left[ \, \tilde{\e}_{t,K} \cdot \tilde \omega_{t+1} \mid \mathcal{F}_t \right] = 0, \qquad \E \left[ \| \tilde{\e}_{t,K} \|^2 \cdot \tilde \omega_{t+1}^2 \mid \mathcal{F}_t \right] \leq \C' \ \  \ \text{for some constant} \ \C' < \infty.$$
Under Assumption~\ref{cond-stepsize} on the stepsize $\{\alpha_t\}$, this implies by \cite[Lemma 1]{tsi94} that $\wtld W_{t, K}  \asto \0$.

Next we show $\lim_{t \to \infty} \E \big[  \| \wtld W_{t,K} \| \big] = 0$. Since $\alpha_t \in (0,1]$ for all $t$ and $\wtld W_{0,K}=\0$, for each $t \geq 0$, $\wtld W_{t+1,K}$ can be expressed as a convex combination of $\e_{j,K} \, \tilde\omega_{j+1}, j \leq t$, with coefficients $c_{t,j}$ (each $c_{t,j}$ is a function of $(\alpha_0, \ldots, \alpha_t)$).
Consequently,
$\| \wtld W_{t+1, K} \| \leq \sum_{j=0}^{t} c_{t,j} \| \e_{j,K} \| \cdot |\tilde\omega_{j+1} |,$
and by the convexity of the function $x^2$, 
$$\| \wtld W_{t+1, K} \|^2 \leq \textstyle{\sum}_{j=0}^{t} c_{t,j} \| \e_{j,K}\|^2 \cdot |\tilde \omega_{j+1} |^2.$$ 
As discussed earlier, the variance of $\e_{j,K} \cdot \tilde \omega_{j+1}$ can be bounded uniformly for all $j$,  so the preceding inequality implies that there exists some constant $\C'<\infty$ with
\begin{equation} \label{eq-prf-noisyR1}
  \E \big[ \| \wtld W_{t+1, K} \|^2 \big] \leq \C', \qquad \forall \, t \geq 0.
\end{equation}  
This in turn implies that the sequence $\{ \| \wtld W_{t,K} \|, t \geq 0 \}$ is uniformly integrable \citep[see e.g.,][p.\ 32]{Bil68}; i.e., 
\begin{equation} 
   \sup_{t \geq 0} \, \E \Big[ \| \wtld W_{t, K}  \| \cdot \mathbb{1}\big( \| \wtld W_{t, K} \| \geq a \big) \Big] \to 0 \quad \text{as} \ \ a \to + \infty, \notag
\end{equation}   
(where $\mathbb{1}\big( \| \wtld W_{t, K} \| \geq a \big)$ is the indicator for the event $\| \wtld W_{t, K} \| \geq a$).
By \cite[Lemma IV-2-5, p.\ 66]{Nev75}, every uniformly integrable sequence of random variables which converges almost surely also converges in $L^1$. Therefore, since $\wtld W_{t,K} \asto \0$ as proved earlier, we have $\lim_{t \to \infty} \E \big[  \| \wtld W_{t,K} \| \big] = 0$.

We now prove $\lim_{t \to \infty} \E \big[  \| W_{t} \| \big] = 0$ similarly to the proof step (iii) for Theorem~\ref{thm-l1conv}. 
For each $K \geq 1$, since $\lim_{t \to \infty} \E \big[  \| \wtld W_{t,K} \| \big] = 0$, we have
\begin{equation} 
  \limsup_{t \to \infty} \E \left[  \big\| W_{t} \big\| \right] \leq \limsup_{t \to \infty} \E \left[  \big\| W_{t}  - \wtld W_{t,K} \big\| \right] + \limsup_{t \to \infty} \E \left[  \big\| \wtld W_{t,K} \big\| \right] = \limsup_{t \to \infty} \E \left[  \big\| W_{t}  - \wtld W_{t,K} \big\| \right]. \notag
\end{equation}  
Thus it is sufficient to prove that 
\begin{equation} \label{eq-prf-noisyR2}
 \lim_{K \to \infty} \limsup_{t \to \infty} \E \left[  \big\| W_{t}  - \wtld W_{t,K} \big\| \right] = 0.
\end{equation} 

To this end, let us write
$$ W_{t+1} - \wtld W_{t+1,K} = (1 - \alpha_t) \big( W_{t} - \wtld W_{t,K} \big) + \alpha_t \big( \e_t - \tilde{\e}_{t,K} \big) \cdot \tilde \omega_{t+1}.$$
Since the number of states and actions is finite, Eq.~(\ref{eq-cond-noise}) implies that for all $t$, $\E \big[ | \tilde \omega_{t+1} | \mid \mathcal{F}_t \big] \leq \C'$ for some constant $\C' < \infty$.
Consequently,
\begin{align*}
 \E \left[ \big\| W_{t+1} - \wtld W_{t+1,K} \big\| \right] & \leq (1 - \alpha_t)  \, \E \left[ \big\| W_{t} - \wtld W_{t,K} \big\| \right] + \alpha_t \, \C' \cdot \E \left[ \big\| \e_t - \tilde{\e}_{t,K} \big\| \right]   \\
    & \leq (1 - \alpha_t) \, \E \left[ \big\| W_{t} - \wtld W_{t,K} \big\| \right] + \alpha_t \, \C' \cdot \C_K,
\end{align*}
where the second inequality follows from Prop.~\ref{prp-3}(i), and $\C_K, K \geq 1$, are constants with the property that $\C_K \downarrow 0$ as $K \to  \infty$. 
By Assumption~\ref{cond-stepsize} on the stepsize, the preceding inequality implies that for each $K$,
$$ \limsup_{t \to \infty} \E \left[ \big\|  W_{t} - \wtld{W}_{t,K} \big\| \right]  \leq \C' \cdot \C_K.$$
Letting $K$ go to infinity and using the fact $\C_K \downarrow 0$, we obtain the desired equality (\ref{eq-prf-noisyR2}), which implies 
$\lim_{t \to \infty} \E \left[ \| W_t  \| \right] = 0$ as discussed earlier.
\end{proofof}

\begin{proofof}{Prop.~\ref{prp-noise}(ii)}
Note that with $\alpha_t = 1/(t+1)$, the convergence $W_t \asto \0$ we want to prove is equivalent to the convergence of the time average, $\tfrac{1}{t+1} \sum_{k=0}^t \e_k \cdot \rho_k \omega_{k+1} \asto \0$, where each term in the sum is a function of $(\e_k, S_k, A_k, S_{k+1}, R_{k})$: 
$$ \e_k \cdot \rho_k \omega_{k+1}  = \e_k \cdot \rho(S_k, A_k) \cdot \big( R_k - r(S_k,A_k,S_{K+1}) \big).$$

By Theorem~\ref{thm-erg}, the Markov chain $\{Z_t\}=\{(S_t, A_t, \e_t, \F_t)\}$ has a unique invariant probability measure $\zeta$. Consequently, the Markov chain $\{Z'_t\} := \{(S_t, A_t, \e_t, \F_t, S_{t+1}, R_{t})\}$ has a unique invariant probability measure $\zeta'$, determined by $\zeta$ together with the probabilities of the successor state $S_{t+1}$ given $(S_t, A_t)$ and the conditional distribution of the reward $R_t$ given $(S_t, A_t, S_{t+1})$, which are specified by the MDP model.
Let $\E_{\zeta'}$ denote expectation with respect to the probability distribution of the stationary Markov chain $\{Z'_t\}$ 
with the initial distribution being $\zeta'$.
From Theorem~\ref{thm-erg}(ii) and the relation between $\zeta'$ and $\zeta$, we have 
\begin{equation} \label{eq-prf-noisyR3}
 \E_{\zeta'} \big[ \| \e_0 \| \cdot  \rho_0 \, | \omega_1 | \big] \leq \C' \E_{\zeta} \big[ \| \e_0 \| \big]  < \infty
\end{equation} 
for some constant $\C' < \infty$. Specifically, in the above, we obtain the first inequality by bounding the conditional expectation of $|\omega_1|$, conditioned on $(\e_0, S_0, A_0)$, by some finite constant [cf.\ Eq.~(\ref{eq-cond-noise})], and we then obtain the second inequality by applying Theorem~\ref{thm-erg}(ii).

Given the finite expectation in Eq.~(\ref{eq-prf-noisyR3}), we can apply to the stationary process $\{Z'_t\}$ with initial distribution $\zeta'$ a strong law of large numbers for stationary processes [\citep[Chap.\ X, Theorem 2.1]{Doob53}; see also \cite[Theorem 17.1.2]{MeT09}]. By this theorem, there exists a nonempty subset $D_1$ of the state space of $\{Z'_t\}$ such that:
\begin{enumerate}
\item[(i)] $D_1$ has $\zeta'$-measure $1$, and 
\item[(ii)] for each initial condition $Z'_0=z' \in D_1$,  
$\{W_t\}$ converges a.s.\ (with respect to the probability measure induced by the initial condition $Z'_0=z'$ for the process $\{Z'_t\}$).
\end{enumerate}
In view of the dependence relations between the variables $(S_t, A_t, S_{t+1}, R_t)$ given by the MDP model, and also in view of the finiteness of the state-action space $\S \times \A$ and the irreducibility property of the behavior policy $\pi^o$ (Assumption~\ref{cond-bpolicy}(ii)), the preceding properties of the set $D_1$ imply that there exists a nonempty subset $D_2$ of the space $\re^{n+1}$ (which is the space of $(\e_t, \F_t)$) such that: 
\begin{enumerate}
\item[(i)] $D_2$ has measure $1$ with respect to the marginal on $\re^{n+1}$ of the invariant probability measure $\zeta$, and
\item[(ii)] for each initial condition $(\e_0, \F_0, S_0) = (\e, \F, s) \in D_2 \times \S$, $\{W_t\}$ converges a.s.
\end{enumerate}
But the limit of $\{W_t\}$ cannot differ from $\0$ by Prop.~\ref{prp-noise}(i) proved earlier (since for the given initial condition, $E[\|W_t\|] \to 0$ implies the existence of a subsequence of $\{W_t\}$ converging to $\0$ a.s.\ \cite[Theorem 9.2.1]{Dud02}). Thus, we conclude that for each initial condition $(\e, \F, s) \in D_2\times \S$, $W_t \asto \0$.

Now to establish Prop.~\ref{prp-noise}(ii), we only need to show that for any given initial condition $(\e_0, \F_0) =(\e,\F) \not\in D_2$, $W_t \asto \0$ as well.  
To prove this, let $s \in \S$ be an arbitrary given state. Consider the sequence $\{(\e_t, \F_t, W_t)\}$ with $(\e_0, \F_0)= (\e,\F)$ and $S_0 = s$. 
Consider also a second sequence $\{(\tilde \e_t, \tilde \F_t, \wtld W_t)\}$ which is generated by the same recursion and the same trajectory of states, actions and rewards that define the first sequence $\{(\e_t, \F_t, W_t)\}$, but with a  pair of initial traces $(\tilde \e_0, \tilde \F_0)=(\tilde \e,\tilde \F) \in D_2$, possibly different from $(\e,\F)$.
By what we proved earlier, $\wtld W_t \asto \0$. On the other hand, by definition
$$ W_{t+1} - \wtld{W}_{t+1} = \frac{1}{t+1} \sum_{k = 0}^t \big( \e_k - \tilde{\e}_k \big) \cdot \rho_k  \omega_{k+1}.$$ 
In the right-hand side, we have $\e_k - \tilde{\e}_k \asto \0$ (as $k \to \infty$) by Prop.~\ref{prp-2}. 
Using this and \cite[Lemma 1]{tsi94}, we can conclude that $W_{t+1} - \wtld{W}_{t+1} \asto \0$.
\footnote{To see this, for each integer $K \geq 1$, define a stopping time $\bar t(K) : = \min \{ t \geq 0 \mid \| \e_t - \tilde{\e}_t \| > K \}$ with $\bar t(K) = + \infty$ if $\| \e_t - \tilde{\e}_t \| \leq K$ for all $t$. Let $x_{t;K} = \big( \e_t - \tilde{\e}_t \big) \cdot \rho_t  \omega_{t+1}$ for $t < \bar t(K)$ and $x_{t;K} = \0$ for all $t \geq \bar t(K)$. Define random variables $X_{t+1; K}:= \tfrac{1}{t+1}  \sum_{k = 0}^t x_{k;K}$, $t \geq 0$, which satisfy the recursion
$X_{t+1;K} = (1 - \tfrac{1}{t+1}) \, X_{t;K} + \tfrac{1}{t+1} \, x_{t;K}$ with $X_{0;K}=\0$.
We then apply \cite[Lemma 1]{tsi94} to obtain $X_{t;K} \asto \0$ for each $K$. (To apply the latter lemma, observe that since $\S$ and $\A$ are finite spaces, $\rho_t$ is bounded for all $t$, and then conditioned on $\mathcal{F}_t$, by construction $\{x_{t;K}\}_{t \geq 0}$ have conditional zero means and uniformly bounded variances, similar to the properties of $\{\omega_t\}$ shown in Eq.~(\ref{eq-cond-noise}). So the conditions of \cite[Lemma 1]{tsi94} are satisfied and its conclusion applies.) Consider now the set $\Omega'$ of those sample paths on which, as $t \to \infty$, $\e_t - \tilde{\e}_t \to \0$ and $X_{t;K} \to \0$ for all $K \geq 1$. Note that $\Omega'$ has probability $1$. Now for each sample path from $\Omega'$, since $\e_t - \tilde{\e}_t \to \0$, there exists some $K$ such that $\bar t(K) = +\infty$ and consequently, $W_{t} - \wtld{W}_{t}$ coincides with $X_{t;K}$ for all $t$, which implies $W_{t} - \wtld{W}_{t} \to \0$.}
Since $\wtld{W}_t \asto \0$, this implies $W_t \asto \0$. The proof is now complete. 
\end{proofof}

\subsection{Proof of Theorem~\ref{thm-lstd} on the Convergence of ELSTD($\lambda$)} \label{appsec-prflstd}

The proof proceeds by calculating the limit $\G^*$ in Theorem~\ref{thm-l1conv} for the two functions $h_1, h_2$ in Eq.~(\ref{eq-h-lstd}): 
with $y = (\e, \F) \in \re^{n+1}$,
\begin{equation} \label{app-eq-hlstd}  
h_1(y, s, a, s') =  \e \cdot \rho(s,a) \, \big( \gamma(s') \fe(s')^\top - \fe(s)^\top \big), \quad  h_2(y, s, a, s') =  \e \cdot \rho(s,a) \, r(s,a,s'),
\end{equation}
which are associated with the ELSTD($\lambda$) iterates $\bA_t, \bb_t$, respectively. Specifically, based on the proof of Theorem~\ref{thm-l1conv}, we first calculate for each $K$, the limit $\G^*_K$ given in Eq.~(\ref{eq-prf4a}), which is associated with the truncated traces $(\tilde{\e}_{t,K}, \tilde{\F}_{t,K})$. We then take $K$ to $\infty$ to get the expression of $\G^*$ since $\G^* = \lim_{K \to \infty} \G^*_K$, as shown in the step (ii) of the proof of Theorem~\ref{thm-l1conv}. The details of this calculation are given below, and the subsequent Lemma~\ref{lma-lstd} establishes that 
\begin{equation} \label{eq-lstd-G*}
   \G^* = \bA  \ \ \ \text{for} \ h = h_1; \qquad \G^* = \bb \ \ \  \text{for} \ h = h_2.
\end{equation}   
Let us give now the rest of the proof of Theorem~\ref{thm-lstd}, assuming for the moment that Eq.~(\ref{eq-lstd-G*}) has been proved.
Then, with $h = h_1$, Theorem~\ref{thm-l1conv} yields the $L^1$-convergence of $\{\bA_t\}$ to $\bA$, and Theorem~\ref{thm-asconv} yields $\bA_t \asto \bA$ for stepsizes $\alpha_t = 1/(t+1)$.

For the iterates $\{\bb_t\}$ [cf.\ Eq.~(\ref{eq-lstd2})], we also need to take care of the noise in the rewards $R_t$, by using Prop.~\ref{prp-noise}. Specifically, with $W_0=\0$, let
$$ \omega_{t+1} = R_{t} - r(S_t, A_t, S_{t+1}), \qquad W_{t+1} = (1 - \alpha_t) \, W_t + \alpha_t \, \e_t \, \rho_t \cdot \omega_{t+1}, \qquad t \geq 0,$$
[cf.\ Eq.~(\ref{eq-noiseW})]. By definition,
\begin{align*}
   \bb_{t+1}  & =  ( 1 - \alpha_t) \, \bb_t + \alpha_t \,  \e_t \cdot \rho_t \, R_{t}  
         = ( 1 - \alpha_t) \, \bb_t + \alpha_t \,  \e_t \cdot \rho_t \, \big( r(S_t, A_t, S_{t+1}) +  \omega_{t+1} \big),
\end{align*}  
so the iteration for $\{\bb_t\}$ can be equivalently expressed as
$$  \bb_{t+1}  = \G_{t+1} + W_{t+1},$$
where $\G_{t+1}$ is given by the recursion~(\ref{eq-G}) with $h=h_2$ and $\G_0 = \bb_0$, and $W_{t+1}$ is as defined above.
Then by Theorem~\ref{thm-l1conv}, Eq.~(\ref{eq-lstd-G*}) and Prop.~\ref{prp-noise}(i), we have
$$ \lim_{t \to \infty} \E \big[ \big\| \bb_{t} - \bb \big\| \big] \leq \lim_{t \to \infty} \E \big[ \big\| \G_{t} - \G^* \big\| \big] + \lim_{t \to \infty} \E \big[ \big\| W_{t}  \big\| \big]  = 0.$$
This proves the $L^1$-convergence of $\{\bb_t\}$ to $\bb$. 
Similarly, its a.s.\ convergence in the second part of Theorem~\ref{thm-lstd} follows from Theorem~\ref{thm-asconv}, Eq.~(\ref{eq-lstd-G*}) and Prop.~\ref{prp-noise}(ii) as
$$ \G_{t} \asto \G^* = \bb  \ \ \text{and} \ \  W_t \asto \0 \qquad \Longrightarrow \qquad \bb_t = \G_t + W_t \asto \bb.$$
Thus Theorem~\ref{thm-lstd} is proved.

In the rest of this subsection, we verify Eq.~(\ref{eq-lstd-G*}), which we used in the proof above.

\subsection*{Computing the Limiting Matrix and Vector for ELSTD($\lambda$)} 

The desired limits for ELSTD($\lambda$) are the matrix $\bA$ and vector $\bb$ given in Eqs.~(\ref{eq-bellman-def})-(\ref{eq-Ab}), Section~\ref{sec-alg}:
\begin{align}
   \bA  &   =  - \Fe^\top  \bM \, (I - \PL) \,  \Fe   \notag  \\
    & =  -  \Fe^\top  \bM \, (I - \P \Gm \Lm)^{-1} \, (I - \P \Gm) \,  \Fe,   \label{appeq-A} \\
    \bb  & =  \Fe^\top  \bM \,  \rl = \Fe^\top  \bM \, (I - \P \Gm \Lm)^{-1} \, \r, \label{appeq-b}
\end{align}   
where $\bM$ is a diagonal matrix with
$$ diag(\bM)= \bi^{\top} \, (I - \PL)^{-1}, \qquad  \bi \in \re^N, \  \bi(s) = d_{\pi^o}(s) \cdot  \i(s), \ s \in \S,$$
and $d_{\pi^o}(s)$ is the steady state probability of state $s$ under the behavior policy $\pi^o$ (cf.\ Assumption~\ref{cond-bpolicy}(ii)), and $\i(s)$ is the ``interest'' weight for state $s$.
By the definition (\ref{eq-bellman-def}) of $\PL$, we can also write
\begin{equation} \label{appeq-M}
   diag(\bM)= \bi^{\top} \, (I - \P \Gm)^{-1} \, (I - \P \Gm \Lm).
\end{equation}   

Recall that $\E_{\zeta}$ denotes expectation with respect to the stationary Markov chain $\{\Z_t\}$, where $\Z_t = (S_t, A_t, \e_t, \F_t)$, with its unique invariant probability measure $\zeta$ as the initial distribution (cf.~Theorem~\ref{thm-erg}). We denote $Y_t = (\e_t, \F_t)$.

\begin{lem} \label{lma-lstd}
Under Assumption~\ref{cond-bpolicy},
\begin{align*} 
     \E_{\zeta} \big[ h_1(Y_0, S_0, A_0, S_1) \big] =   \bA, \qquad
     \E_{\zeta} \big[ h_2(Y_0, S_0, A_0, S_1) \big]  =  \bb.
\end{align*}   
\end{lem}

\begin{proof}
The proof proceeds as follows. We first calculate the limit vector $\G^*_K$ defined in Eq.~(\ref{eq-prf4a}) in the proof of Theorem~\ref{thm-l1conv}, for the two choices of the function $h$ given in Eq.~(\ref{app-eq-hlstd}): $h=h_1$, $h=h_2$. We then calculate $\G^*$ by its definition given in the proof of Theorem~\ref{thm-l1conv}: $\G^* = \lim_{K \to \infty} \G^*_K$. 
Theorem~\ref{thm-asconv} shows $\E_{\zeta} \big[ h(Y_0, S_0, A_0, S_1) \big]  = G^*$, so the lemma follows if we prove that $\G^* = \bA$ for $h = h_1$, and $\G^* = \bb$ for $h = h_2$.

Let $\E_0$ denotes expectation with respect to the probability measure of the stationary Markov chain $\{(S_t, A_t)\}$.  
Recall that for each $K \geq 1$, $\G^*_K$ is defined by Eq.~(\ref{eq-prf4a}) as 
\begin{equation} \label{appeq-prf-lmt1}
   \G^*_K = \E_0 \big[ \, h (Y_{t,K}, S_t, A_t, S_{t+1})  \,\big], \qquad   \forall \, t > 2 K, 
\end{equation}   
where $Y_{t,K} = (\tilde{\e}_{t,K}, \tilde{\F}_{t,K})$ are the truncated traces defined together with $\tilde{\M}_{t,K}$ in Eqs.~(\ref{eq-tF})-(\ref{eq-te}).
To simplify the calculation of $\G^*_K$, we first calculate $\E_0 \big[ \tilde{\M}_{t,K} \Psi(S_t) \big]$ for any given $t > K$ and any matrix or vector-valued function $\Psi$ on $\mathcal{S}$.

By Eqs.~(\ref{eq-tF})-(\ref{eq-tM}), 
$$\tilde{\M}_{t,K} = \lambda_t \, \i(S_t) + (1 - \lambda_t) \, \tilde{\F}_{t,K}, \qquad  \tilde{\F}_{t,K}  = \sum_{k=t-K}^t \i(S_k) \cdot \big(\rho_{k} \gamma_{k+1} \cdots \rho_{t-1} \gamma_t \big).$$
Thus, $\tilde{M}_{t,K}$ can be equivalently expressed as
$$ \tilde{\M}_{t,K} = \i(S_t) + \sum_{k=1}^{K} \i(S_{t-k}) \cdot \big(\rho_{t-k} \gamma_{t-k+1} \cdots \rho_{t-1} \gamma_{t} \big) \cdot ( 1 - \lambda_t).$$ 
To calculate $\E_0 \big[ \tilde{\M}_{t,K} \Psi(S_t) \big]$, we calculate the expectation for each term in the above summation separately.

In what follows, for each $s \in \mathcal{S}$, let $\mathbb{1}_s(\cdot)$ denote the indicator for state $s$. For an expression $H$ that results in an $N$-dimensional vector, we write $(H)(s)$ for the $s$-th entry of the resulting vector.
Under Assumption~\ref{cond-bpolicy}(ii), we have
$$ \E_0 \big[ \, \i(S_t) \cdot \mathbb{1}_s(S_t) \big] = d_{\pi^o}(s) \, \i(s) = (\,\bi^\top \, I)(s),$$
and for $k = 1, 2, \ldots, K$,
$$ \E_0 \big[ \, \i(S_{t-k}) \cdot \big(\rho_{t-k} \gamma_{t-k+1} \cdots \rho_{t-1} \gamma_{t} \big) \cdot ( 1 - \lambda_t) \cdot \mathbb{1}_s(S_t) \, \big] 
    = \left( \bi^{\top} (\P \Gm)^{k} (I - \Lm) \right)(s).$$
Hence
$$ \E_0 \big[ \,\tilde{\M}_{t,K} \cdot \mathbb{1}_s(S_t)  \,\big] = \left( \bi^{\top} \Big( I + \sum_{k=1}^K (\P \Gm)^{k} (I - \Lm) \Big) \right)(s),$$
and consequently,
\begin{align}
    \E_0 \big[ \, \tilde{\M}_{t,K} \cdot \Psi(S_t) \, \big] & = \sum_{s \in \mathcal{S}}  \E_0 \big[ \tilde{\M}_{t,K} \cdot \mathbb{1}_s(S_t) \big] \cdot \Psi(s) \notag \\
    & = \sum_{s \in \mathcal{S}}  \left( \bi^{\top} \Big( I + \sum_{k=1}^K (\P \Gm)^{k} (I - \Lm) \Big)\right) (s) \cdot \Psi(s). \label{eq-prf5a}
\end{align}    

Let us now calculate $\G^*_K$ for $h=h_1$ or $h_2$ simultaneously, using Eq.~(\ref{appeq-prf-lmt1}) and the expressions of $h_1, h_2$ given in Eq.~(\ref{app-eq-hlstd}).
Let $t > 2K$, and let $\mathcal{F}_t =\sigma(S_0, A_0, \ldots, S_t)$.
For the term appearing after $\e$ in the expression of $h_1$, we have
$$ \E_0 \big[ \, \rho_t \big( \gamma_{t+1} \fe(S_{t+1})^\top - \fe(S_t)^\top \big) \mid \mathcal{F}_t \big] = \Psi_1(S_t), 
$$
where $\Psi_1$ maps each $s$ to the $s$-th row of the matrix $(\P \Gm - I) \Fe$.
For the term appearing after $\e$ in the expression of $h_2$, we have
$$ \E_0 \big[ \, \rho_t \, r(S_t, A_t, S_{t+1}) \mid \mathcal{F}_t \big] =\Psi_2(S_t), $$ 
where $\Psi_2$ maps each $s$ to the $s$-th entry of the vector $\r$.
Denote $\Psi_1^o= (\P \Gm - I) \Fe, \Psi_2^o=\r$.

Corresponding to $h=h_1$ or $h_2$, let $\Psi = \Psi_1$ or $\Psi_2$, and let $\Psi^o = \Psi_1^o$ or $\Psi_2^o$, respectively.
Then, by Eqs.~(\ref{appeq-prf-lmt1}) and~(\ref{app-eq-hlstd}), we have
\begin{align}
 \G^*_K & = \E_0 \big[ \, \E_0 [ \, h (Y_{t,K}, S_t, A_t, S_{t+1}) \mid \mathcal{F}_t \, ] \,\big]  \notag \\
 & =  \E_0 \left[ \, \tilde{\e}_{t,K} \cdot  \Psi(S_t) \, \right]  \notag \\
 & = \E_0 \left[ \, \sum_{k=t-K}^t \tilde{\M}_{k,K} \cdot \fe(S_k) \cdot (\beta_{k+1} \cdots \beta_t) \cdot  \Psi(S_t) \, \right]  \label{appeq-prf-lmt2} \\
 & = \sum_{k=t-K}^t \E_0 \left[ \,  \tilde{\M}_{k,K} \cdot \fe(S_k) \cdot \E_0 \left[ \,  (\beta_{k+1} \cdots \beta_t) \cdot  \Psi(S_t)  \mid \mathcal{F}_k \, \right] \, \right]  \notag \\
 & = \sum_{k=t-K}^t \E_0 \left[ \,  \tilde{\M}_{k,K} \cdot \fe(S_k) \cdot \big( (\P \Gm \Lm)^{t-k}  \Psi^o \big)_{\nr}(S_k) \right].  \label{appeq-prf-lmt3} 
\end{align}
In the above we used the definition (\ref{eq-te}) of $\tilde{\e}_{t,K}$ in Eq.~(\ref{appeq-prf-lmt2}), and in Eq.~(\ref{appeq-prf-lmt3}), the term $( \cdots)_{\nr}(S_k)$ inside the expectation denotes the $S_k$-th row of the matrix or vector given by the expression inside the parentheses $(\cdots)$. We shall also use this notational convention in the proof below.

From Eq.~(\ref{appeq-prf-lmt3}), using the fact that the expectation is with respect to the stationary Markov chain $\{(S_t, A_t)\}$, we obtain
\begin{align}
 \G^*_K & = \sum_{k=t-K}^t \E_0 \left[ \,  \tilde{\M}_{t,K} \cdot \fe(S_t) \cdot \big( (\P \Gm \Lm)^{t-k}  \Psi^o \big)_{\nr}(S_t) \right] \notag \\
 & = \E_0 \left[ \,  \tilde{\M}_{t,K} \cdot \fe(S_t) \cdot \Big( \sum_{k=0}^{K} (\P \Gm \Lm)^k \cdot  \Psi^o \Big)_{\nr}(S_t) \right]. \label{eq-prf5b}
\end{align}
Corresponding to the last two terms inside the expectation above, define a function $\Psi_K$ on $\S$ by
$$\Psi_K(s) = \fe(s) \cdot  \left( \sum_{k=0}^{K} (\P \Gm \Lm)^k \cdot  \Psi^o \right)_{\nr}(s), \qquad s \in \mathcal{S}. $$ 
Then by combining Eq.~(\ref{eq-prf5b}) with (\ref{eq-prf5a}), we have
\begin{equation}
  \G^*_K = \E_0 \left[ \,  \tilde{\M}_{t,K} \cdot \Psi_K(S_t) \right] = \sum_{s \in \mathcal{S}}  \left( \bi^{\top} \Big( I + \sum_{k=1}^K (\P \Gm)^{k} (I - \Lm) \Big)\right) (s) \cdot \Psi_K(s). \label{eq-prf6}
 \end{equation}  

We now take $K$ to infinity to get the expression of $\G^*$. 
For $h=h_1$, $\Psi^o$ in the definition of $\Psi_K$ is given by $\Psi^o= (\P \Gm - I) \Fe$. 
Using the equality relations,
$$ I + \sum_{k=1}^\infty (\P \Gm)^k (I - \Lm)  = (I - \P \Gm)^{-1} (I - \P \Gm \Lm), \qquad \quad \sum_{k=0}^{\infty} (\P \Gm \Lm)^k = (I - \P \Gm \Lm)^{-1},$$
we obtain from the expression (\ref{eq-prf6}) of $\G^*_K$ that
\begin{align*}
 \G^* & = \sum_{s \in \mathcal{S}}  \Big( \bi^{\top} (I - \P \Gm)^{-1} (I - \P \Gm \Lm) \Big) (s) \cdot \fe(s) \cdot 
\Big( (I - \P \Gm \Lm)^{-1} \cdot (\P \Gm - I) \Fe \Big)_{\nr}(s)  \\
& = \Fe^\top  \bM \,  (I - \P \Gm \Lm)^{-1} \, (\P \Gm - I) \,  \Fe = \bA,
\end{align*}
where the second equality follows from the expression (\ref{appeq-M}) for $diag(\bM)$, and the third equality follows from Eq.~(\ref{appeq-A}).
For $h = h_2$, $\Psi^o$ in the definition of $\Psi_K$ is given by $\Psi^o= \r$, and a similar calculation gives
\begin{align*}
 \G^* & = \sum_{s \in \mathcal{S}}  \Big( \bi^{\top} (I - \P \Gm)^{-1} (I - \P \Gm \Lm) \Big) (s) \cdot \fe(s) \cdot \big( (I - \P \Gm \Lm)^{-1} \cdot \r \big)_{\nr}(s) \\
        & = \Fe^\top \bM \, (I - \P \Gm \Lm)^{-1} \, \r = \bb,
\end{align*}  
where the last equality follows from Eq.~(\ref{appeq-b}).
\end{proof}

\subsection{Related Result: Alternative Proof of Existence of an Invariant Probability Measure} \label{appsec-invm-alt}

Consider the Markov chain $\{\Z_t\}$, where $\Z_t = (S_t, A_t, \e_t, \F_t)$. 
In Section~\ref{sec-elstd}, we used, among others, the weak Feller property of the Markov chain $\{\Z_t\}$ to establish the existence of at least one invariant probability measure for $\{\Z_t\}$ (the property (iv) in Section~\ref{sec-elstd1}). We now give an alternative proof for this statement, by constructing directly an invariant probability measure. This proof is similar to that of \cite[Lemma 4.2]{Yu-siam-lstd}, and it was motivated by an analysis of the LSTD($1$) algorithm by~\citet[Chap.\ 11.5.2]{Mey08}. The proof will also yield directly that under that invariant probability measure, $\big\| (\e_0, \F_0) \big\|$ has a finite expectation, which was established in Theorem~\ref{thm-erg}(ii) earlier by using different arguments. 

\begin{prop} \label{prp-alt-invm}
Under Assumption~\ref{cond-bpolicy}, the Markov chain $\{Z_t\}$ has at least one invariant probability measure $\zeta$ with $\E_{\zeta} \big[ \big\| (\e_0, \F_0) \big\| \big] < \infty$.  
\end{prop}

\begin{proof}
Consider a double-ended stationary Markov chain $\{(S_t, A_t) \mid - \infty < t < \infty \}$ with transition probability matrix $\Pb$ and probability distribution $\Po$. 
In this proof, let $\E_0$ denote expectation with respect to $\Po$.
Let $X_t = \big( (S_t, A_t), (S_{t-1}, A_{t-1}), \ldots \big)$, and denote by $P_X$ the probability distribution of $X_t$, which is a probability measure on $(\mathcal{S} \times \mathcal{A})^\infty$ and is the same for all $t$ due to stationarity. 

We will first define two functions, $f: (\mathcal{S} \times \mathcal{A})^\infty \to \re_+$ and $\psi: (\mathcal{S} \times \mathcal{A})^\infty \to \rn$, which relate to the traces $\F, \e$, respectively. We will then show that the distribution of $\big(S_0, A_0, \psi(X_0), f(X_0)\big)$ is an invariant probability measure of $\{\Z_t\}$.

Let us introduce some notation. For $x \in (\mathcal{S} \times \mathcal{A})^\infty$, we index the components of $x$ as $x = \big( (s_0, a_0), (s_{-1}, a_{-1}), \ldots \big)$, and we denote by $x^{(-k)}$ the tail of $x$ starting from $s_{-k}$, i.e., $x^{(-k)}  = \big( (s_{-k}, a_{-k}), (s_{-k-1}, a_{-k-1}), \ldots \big).$ 
Recall that $\beta_t = \rho_{t-1} \gamma_t \lambda_t$. Correspondingly, we define a function $\beta: \mathcal{S} \times \mathcal{A} \times \mathcal{S} \to \re_+$ by
$$ \beta(s,a,s') = \rho(s,a) \, \gamma(s') \, \lambda(s').$$
For an expression $H$ that results in a vector in $\re^N$, we write $(H)(s)$ for the $s$-th entry of that vector.

We now define the function $f$.\footnote{We note that to gain intuition about the proof, it will be helpful to compare our definition of $f$ with the expression of $\F_t$ in Eq.~(\ref{eq-F}), and compare our subsequent definition of $\psi$ with the expression of $\e_t$ in Eq.~(\ref{eq-e}).} 
Since 
\begin{align}
  \sum_{k=0}^\infty \E_0 \big[ \i(S_{-k}) \cdot \big( \rho_{-k} \gamma_{-k+1} \cdots \rho_{-1} \gamma_0 \big) \big] 
  & = \sum_{k=0}^\infty \E_0 \big[ \i(S_{-k})  \cdot \big( (\P \Gm)^k \1 \big)(S_{-k}) \big] \notag \\
  & =  \bi^\top \sum_{k=0}^\infty (\P \Gm)^k \1 < \infty, \label{eq-prf-inv}
\end{align}  
we can define a nonnegative real-valued measurable function $f$ on $(\mathcal{S} \times \mathcal{A})^\infty$ such that
$$ f(x) =  \begin{cases}   \sum_{k=0}^\infty \i(s_{-k}) \cdot \big( \rho(s_{-k}, a_{-k}) \gamma(s_{-k+1}) \cdots \rho(s_{-1},a_{-1}) \gamma(s_0) \big), \quad & \text{if} \ x \in D_1, \\
 0, & \text{otherwise}, \end{cases}$$
where $D_1$ is a subset of  $(\mathcal{S} \times \mathcal{A})^\infty$ with $P_X(D_1) = 1$.
By Eq.~(\ref{eq-prf-inv}),
$$ \int f(x) P_X(dx) = \E_0 \big[ f(X_0)\big] = \sum_{k=0}^\infty \E_0 \big[ \i(S_{-k}) \cdot \big( \rho_{-k} \gamma_{-k+1} \cdots \rho_{-1} \gamma_0 \big) \big] < \infty.$$

We now define the function $\psi$. Define two constants $L', L$ as follows.
Let $\C'=\E_0 \big[ f(X_0)\big]$; equivalently, $\C'=\E_0 \big[ f(X_{-k})\big]$ for all $k$ by stationarity. 
Let $\C \geq  \max\big\{ \i(s), \|\fe(s) \| \big\}$ for all $s \in \mathcal{S}$.  
By taking conditional expectation similarly to the proof of Lemma~\ref{lma2}, we have the following bound:
\begin{align*}
& \sum_{k=0}^\infty \E_0 \Big[ \Big\|   \big( \lambda_{-k} \cdot \i(S_{-k}) + ( 1 - \lambda_{-k} ) f(X_{-k}) \big) \cdot \fe(S_{-k})  \cdot \big(\beta_{-k+1} \cdots \beta_0 \big) \Big\| \Big] \\
   =  \, & \sum_{k=0}^\infty \E_0 \Big[  \big( \lambda_{-k} \cdot \i(S_{-k}) + ( 1 - \lambda_{-k} ) f(X_{-k}) \big) \cdot \big\| \fe(S_{-k}) \big\| \cdot \big(\beta_{-k+1} \cdots \beta_0 \big) \Big] \\
    \leq  \, & \, (\C+\C') \cdot \C \cdot \sum_{k=0}^\infty \1^\top (\P \Gm \Lm)^{k} \1 < \infty.
\end{align*}  
Therefore, by a theorem on integration \citep[Theorem 1.38, p.~28-29]{Rudin66}, we can define a measurable function $\psi$ on $(\mathcal{S} \times \mathcal{A})^\infty$ such that the following hold: 
\begin{enumerate}
\item[(i)] on a set $ D_2 \subset (\mathcal{S} \times \mathcal{A})^\infty$ with $P_X(D_2) = 1$,
\begin{align*}
 \psi(x)  =  
 \sum_{k=0}^\infty  \, & \big( \lambda(s_{-k}) \, \i(s_{-k}) + ( 1 - \lambda(s_{-k}) ) \, f(x^{(-k)}) \big) \cdot \fe(s_{-k}) \\
   &  \cdot \big(\beta(s_{-k}, a_{-k}, s_{-k+1})  \cdots \beta(s_{-1}, a_{-1}, s_0) \big), 
\end{align*}  
where the infinite series on the right-hand side converges to a vector in $\rn$;
\item[(ii)] outside $D_2$, $\psi(x) =\0$; and 
\item[(iii)] $\psi(x)$ is integrable with
$$ \E_0 \big[ \| \psi(X_0) \|  \big] = \int \| \psi(x) \| \, P_X(dx) < \infty$$
and
\begin{align*}
\int \psi(x) \, P_X(dx) & = \sum_{k=0}^\infty \E_0 \Big[    \big( \lambda_{-k} \i(S_{-k}) + ( 1 - \lambda_{-k} ) f(X_{-k}) \big) \cdot \fe(S_{-k})  \cdot \big(\beta_{-k+1} \cdots \beta_0 \big)  \Big]. 
\end{align*}
\end{enumerate}

Let $Y_0^o = (\psi(X_0), f(X_0))$. We now show that the probability distribution of $(S_0, A_0, Y_0^o)$ is an invariant probability measure of the Markov chain $\{Z_t\}$. To this end, consider $X_1 = \big( (S_1, A_1), X_0 \big)$, and let us define $Y_1^o=(\e^o_1,\F^o_1)$ based on $Y_0^o$ and $(S_0, A_0, S_1)$, using the same recursion that defines $(\e_t,\F_t)$ [cf.~Eqs.~(\ref{eq-td3})-(\ref{eq-td1})]:
$$ \F^o_1 = \gamma_1  \, \rho_0  \cdot f(X_0) + \i(S_1), \qquad \e_1^o =  \beta_1 \, \psi(X_0) +  \big( \lambda_{1} \, \i(S_{1}) + ( 1 - \lambda_{1} ) \, \F^o_1 \big) \cdot \fe(S_1).$$
If $(S_0, A_0, Y_0^o)$ and $(S_1, A_1, Y_1^o)$ have the same distribution, then this distribution must be an invariant probability measure of $\{Z_t\}$ because the stochastic kernel that governs the transition from $(S_0, A_0, Y_0^o)$ to $(S_1, A_1, Y_1^o)$ is the same as that from $Z_0=\big(S_0,A_0, (\e_0,\F_0)\big)$ to $Z_1=\big(S_1,A_1, (\e_1,\F_1)\big)$.

Now define a set $D \subset (\S \times \A)^\infty$ by $D = D_1 \cap D_2 \cap \big(\mathcal{S} \times \mathcal{A} \times (D_1 \cap D_2 ) \big),$ where $D_1, D_2$ are the sets in the definitions of the functions $f$ and $\psi$, respectively. 
Since $P_X(D_1)=P_X(D_2) = 1$, we have 
$$P_X(D) = \Po(X_1 \in D) =  \Po(X_1 \in D_1 \cap D_2, X_0 \in D_1 \cap D_2) = 1.$$
Consider the case $X_1 \in D$. Then both $X_0, X_1 \in D_1 \cap D_2$. By the definition of $f$ on $D_1$, it follows that
$$\F_1^o = \i(S_1) +  \sum_{k=0}^\infty \i(S_{-k}) \cdot \big( \rho_{-k} \gamma_{-k+1} \cdots \rho_{-1} \gamma_0 \big) \cdot \rho_0 \gamma_1 = f(X_1),$$ 
and from this and the definition of $\psi$ on $D_2$, it also follows that 
\begin{align*}
 \e_1^o & = \big( \lambda_{1} \, \i(S_{1}) + ( 1 - \lambda_{1} ) \, f(X_1) \big) \cdot \fe(S_1)   \\
  & \quad \ + \sum_{k=0}^\infty \big( \lambda_{-k} \cdot \i(S_{-k}) + ( 1 - \lambda_{-k} ) f(X_{-k}) \big) \cdot \fe(S_{-k})  \cdot \big(\beta_{-k+1} \cdots \beta_0 \big)  \cdot \beta_1  \\
 & = \psi(X_1).
\end{align*} 
By stationarity, $\big(S_0, A_0, \psi(X_0), f(X_0) \big)$ and $(S_1, A_1, \psi(X_1), f(X_1) \big)$ have the same distribution. 
Denote this distribution by $\zeta$.
Since $(\e_1^o, \F_1^o)$ differs from $\big(\psi(X_1), f(X_1) \big)$ only when $X_1 \not\in D$, an event with $\Po$-probability $0$, we conclude that $(S_0, A_0, Y_0^o)$ and $(S_1, A_1, Y_1^o)$ have the same distribution $\zeta$, which is an invariant probability measure of $\{Z_t\}$ as discussed earlier.  Then, from the integrability property of $\psi$ and $f$ shown earlier, we have
$$\E_\zeta \big[ \big\| (\e_0, \F_0) \big\| \big] \leq \E_\zeta \big[ \big\| \e_0 \big\| \big] + \E_\zeta \big[ \F_0 \big] = \E_0 \big[ \big\| \psi(X_0) \big\| \big] + \E_0 \big[  f(X_0) \big] < \infty.$$
This completes the proof.
\end{proof}

\section{Proofs for Section~\ref{sec-etd}} \label{appsec-etd}

In this appendix we prove Lemma~\ref{lma-pode} and Theorem~\ref{thm-ctd} for the constrained ETD($\lambda$) algorithm~(\ref{eq-emtd-const}). We will restate both theorems for convenience. 

Recall that the constrained ETD($\lambda$) calculates $\w_{t}$, $t \geq 0$, all restricted to be in a closed ball with radius $r$, $\H = \{ \w \in \rn \mid \| \w\|_2 \leq r \}$, according to
$$  \w_{t+1} = \Pi_{\H} \Big( \w_t + \alpha_t \, h(\w_t, \xi_t) + \alpha_t \, \e_t \cdot \tilde \omega_{t+1} \Big),$$
where $\tilde \omega_{t+1} = \rho_t \big(R_t - r(S_t, A_t, S_{t+1})\big)$ is noise, $\xi_t = (\e_t, S_t, A_t, S_{t+1})$, and the function $h$ is given by Eq.~(\ref{eq-mfn1})  as
$$   h(\w, \xi) =  \e \cdot \rho(s, a) \, \big( r(s, a, s') + \gamma(s') \, \fe(s')^\top \w - \fe(s)^\top \w \big), \quad \text{for} \ \ \xi = (\e, s, a, s').$$
The ``mean ODE'' associated with this algorithm is the projected ODE (\ref{eq-pode}):
$$  \dot{x} = \bar h(x) + z, \qquad z \in - \mathcal{N}_\H(x),$$
where $\bar h(x) = \bA x + \bb$, $\mathcal{N}_\H(x)$ is the normal cone of $\H$ at $x$, and $z$ is the boundary reflection term that keeps the solution in $\H$ \citep{KuY03}.  The solution of $\bar h(x) = 0$ is denoted $\w^*$; i.e., $\w^* = -\bA^{-1} \bb$.

\lmapode*

\begin{proof}
By the definition of $\w^*$, $\bA \w^* + \bb = 0$. Therefore,
$$ 0 = \langle \w^*, \bA \w^* + b\rangle =  \langle \w^*, \bA \w^* \rangle + \langle \w^*, \bb \rangle \leq   - c \| \w^* \|_2^2 + \|\bb\|_2 \| \w^*\|_2,$$
which implies $\| \w^*\|_2  \leq \bb\|_2 / c < r$, i.e., $\w^*$ lies in the interior of $\H$.

For a point $x$ on the boundary of $\H$, $\| x\|_2 = r$ and the normal cone $\mathcal{N}_\H(x) = \{ a x \mid a \geq 0 \}$. 
Since $r > \| \bb \|_2/c$, we have
$$ \langle x, \bar h(x) \rangle = \langle x, \bA x  \rangle + \langle x, \bb \rangle \leq - c \| x\|_2^2 + \| x\|_2 \| \bb\|_2 =  r \, ( - c \, r + \| \bb\|_2 ) < 0.$$  
This shows that for any $x$ on the boundary of $\H$, $\bar h(x)$ points inside $\H$ and hence at $x$, the boundary reflection term $z \in - \mathcal{N}_\H(x)$ that keeps the solution in $\H$ is the zero vector. Consequently, any solution of the projected ODE (\ref{eq-pode}) in $\H$ is a solution of the ODE (\ref{eq-ode}), which is $x(\cdot) \equiv \w^*$.
\end{proof}

Next we prove Theorem~\ref{thm-ctd}.

\thmctd*

\begin{proof}
The desired conclusions will follow immediately from \cite[Theorem 6.1.1]{KuY03} and Lemma~\ref{lma-pode}, if we can show that the conditions of \cite[Theorem 6.1.1]{KuY03} are met. Relevant here are the conditions A.6.1.1-A.6.1.4 and A.6.1.6-A.6.1.7 in \cite[p.\ 165]{KuY03}. We first adapt these six conditions to our problem, and by using stronger forms of the conditions A.6.1.6-A.6.1.7 given in \cite[Eq.~(6.1.10), p.\ 166]{KuY03}, we obtain the conditions (i)-(vi) below. 

The first two conditions are for the functions $h, \bar h$ [cf.\ Eqs.~(\ref{eq-mfn1}),~(\ref{eq-ode})] and the noise $\{\tilde \omega_{t}\}$:
\begin{enumerate}
\item[(i)] $\sup_{t \geq 0} \E \big[ \| h(\w_t, \xi_t) + \e_t \cdot \tilde \omega_{t+1} \| \big] < \infty$. 
\item[(ii)] $\bar h(\w)$ is continuous, and $h(\w,\xi)$ is continuous in $\w$ for each $\xi$. 
\end{enumerate}
Condition (i) is satisfied here. Indeed, we have $\sup_{t \geq 0} \E \big[ \| h(\w_t, \xi_t) \| \big] < \infty$, in view of Prop.~\ref{prp-bdtrace}, the Lipschitz continuity of $h$ in $\e$, and the fact that $\|\w_t \|_2 \leq r$ for all $t$ by the definition of the constrained algorithm. 
Since the rewards $R_t$ have bounded variances by assumption and the noise variable $\tilde \omega_{t+1} = \rho_t \big(R_t - r(S_t, A_t, S_{t+1})\big)$ by definition,
we can bound $\E \big[ | \tilde \omega_{t+1} | \mid \mathcal{F}_t \big]$ by some constant for all $t$, where $\mathcal{F}_t = \sigma(S_0, A_0, \ldots, S_{t+1})$, and consequently, we also have  $\sup_{t \geq 0} \E \big[ \| \e_t \cdot \tilde \omega_{t+1} \| \big] < \infty$ by Prop.~\ref{prp-bdtrace}. Hence condition (i) holds.
Condition (ii) is also clearly satisfied here.

The four remaining conditions to be introduced
are of the same type and relate to the asymptotic rate of change conditions introduced by \citep{KuC78}. These conditions can guarantee that the effects caused by the noises $\tilde \omega_{t+1}$ or by the discrepancies between $h$ and $\bar h$ asymptotically ``average out'' so that the desired convergence can take place.

For any real $T' > 0$, define integer $m(T') = \min \{ t \geq 0 \mid \sum_{k = 0}^t \alpha_k > T' \}$. 
Conditions (iii)-(vi) below are required to hold for each $a \geq 0$ and some $T > 0$ (here $a$ and $T$ are real numbers):
\begin{enumerate}
\item[(iii)] For each $\w$,
\begin{equation} \label{eq-prf-thmtd-1}
 \lim_{t \to \infty} \, P \left\{ \sup_{j \geq t} \max_{0 \leq T' \leq T} \, \left\| \sum_{k=m(jT)}^{m(jT+T')-1} \alpha_k \Big( h(\w, \xi_k) - \bar h(\w) \Big) \right\| \geq a \right\} = 0.
 \end{equation} 
 \item[(iv)] 
 \begin{equation} \label{eq-prf-thmtd-2}
 \lim_{t \to \infty} \, P \left\{ \sup_{j \geq t} \max_{0 \leq T' \leq T} \, \left\| \sum_{k=m(jT)}^{m(jT+T')-1} \alpha_k \, \e_k \cdot \tilde \omega_{k+1} \right\| \geq a \right\} = 0.
\end{equation} 
\item[(v)] There exist nonnegative measurable functions $g_1(\w), g_2(\xi)$ such that 
$$\| h(\w,\xi) \| \leq g_1(\w) \, g_2(\xi),$$ 
where $g_1$ is bounded on each bounded set of $\w$, and $g_2$ satisfies that $\sup_{t \geq 0} \E \big[ g_2(\xi_t) \big] < \infty$ and 
\begin{equation} \label{eq-prf-thmtd-3}
 \lim_{t \to \infty} \, P \left\{ \sup_{j \geq t} \max_{0 \leq T' \leq T} \, \left| \sum_{k=m(jT)}^{m(jT+T')-1} \alpha_k  \Big( g_2(\xi_k) - \E \big[ g_2(\xi_k)\big] \Big) \right| \geq a \right\} = 0.
\end{equation} 
\item[(vi)] There exist nonnegative measurable functions $g_3(\w), g_4(\xi)$ such that  for each $\w, \w'$,
$$  \| h(\w, \xi) - h(\w', \xi) \| \leq g_3(\w - \w') \, g_4(\xi),$$
where $g_3$ is bounded on each bounded set of $\w$, with $g_3(\w) \to 0$ as $\w \to 0$, and $g_4$ satisfies that $\sup_{t \geq 0} \E \big[ g_4(\xi_t) \big] < \infty$ and 
\begin{equation} \label{eq-prf-thmtd-4}
 \lim_{t \to \infty} \, P \left\{ \sup_{j \geq t} \max_{0 \leq T' \leq T} \, \left| \sum_{k=m(jT)}^{m(jT+T')-1} \alpha_k  \Big( g_4(\xi_k) - \E \big[ g_4(\xi_k)\big] \Big) \right| \geq a \right\} = 0.
\end{equation} 
\end{enumerate}

One method given in \cite[Chap.~6.2, p.\ 170-171]{KuY03} of verifying the conditions (\ref{eq-prf-thmtd-1})-(\ref{eq-prf-thmtd-4}) above is to show that a strong law of large numbers hold for the processes involved. 
In particular, let $\psi_k$ represent $h(\w, \xi_k) - \bar h(\w)$ for condition (iii),  $\e_k \cdot \tilde \omega_{k+1}$ for condition (iv), $g_2(\xi_k) - \E \big[ g_2(\xi_k)\big]$ for condition (v), and $g_4(\xi_k) - \E \big[g_4(\xi_k) \big]$ for condition (vi). If
\begin{equation} \label{eq-prf-thmtd-5}
  \frac{1}{t+1} \sum_{k=0}^{t} \psi_k \asto 0
\end{equation}  
for the respective $\{\psi_k\}$, then the conditions (\ref{eq-prf-thmtd-1})-(\ref{eq-prf-thmtd-4}) hold for stepsizes satisfying $\alpha_t=O(1/t)$ and $\tfrac{\alpha_t - \alpha_{t+1}}{\alpha_t}= O(1/t)$ \citep[see][Example 6.1, p.\ 171]{KuY03}.

We now apply the convergence results given earlier in this paper to show that the desired convergence (\ref{eq-prf-thmtd-5}) holds for the processes involved in conditions (iii)-(vi).
In particular, for each fixed $\w$, the almost sure convergence part of Theorem~\ref{thm-lstd} implies that
$$ \frac{1}{t+1} \, \sum_{k=0}^t h(\w, \xi_k) \, \asto  \, \E_\zeta \big[ h(\w, \xi_0) \big] = \bar h(\w).$$
Thus, condition (iii) holds, as just discussed. 
By Prop.~\ref{prp-noise}(ii),
$ \frac{1}{t+1} \, \sum_{k=0}^t \e_k \cdot \tilde \omega_{k+1} \asto \0,$
so condition (iv) is also met. 

We verify now conditions (v)-(vi). For condition (v), we take $g_1(\w) = \| \w \| + 1$, and we bound the function $h$ by
$$ \| h(\w, \xi) \| \leq \big( \| \w \| + 1 \big) \, g_2(\xi), \qquad \text{where} \ \ g_2(\xi) = \C \| \e\|, $$
and $\C > 0$ is some constant. (This bound can be verified directly using the expression of $h$ and the fact that the sets $\S$ and $\A$ are finite.) 
Similarly, for condition (vi), we take $g_3(\w) = \| \w\|$, and we bound the change in $h(\w,\xi)$ in terms of the change in $\w$ as follows: 
for any $\w, \w' \in \rn$,
$$  \big\| h(\w, \xi) - h(\w', \xi) \big\| \leq \| \w - \w' \| \, g_4(\xi), \qquad \text{where} \ \ g_4(\xi) = \C' \| \e\|, $$
and $\C' > 0$ is some constant.
Now the functions $g_2, g_4$ are Lipschitz continuous in $\e$. Hence, for $j = 2, 4$, it follows from Prop.~\ref{prp-bdtrace} that
$\sup_{t \geq 0} \E \big[ g_j(\xi_t) \big] < \infty,$
and it follows from Theorems~\ref{thm-asconv} and \ref{thm-l1conv} that
$$  \frac{1}{t+1} \, \sum_{k=0}^t g_j(\xi_k) \, \asto  \, \E_\zeta \big[ g_j(\xi_0) \big], \quad \text{and} \quad \frac{1}{t+1} \, \sum_{k=0}^t \E \big[ g_j(\xi_k) \big] \to \E_\zeta \big[ g_j(\xi_0) \big], \ \ \  \text{as} \ t \to \infty. $$ 
The preceding two relations imply the desired convergence:
$$ \frac{1}{t+1} \, \sum_{k=0}^t \Big( g_j(\xi_k)  - \E \big[ g_j(\xi_k) \big]  \Big) \asto 0, \qquad j = 2, 4.$$
This shows that conditions (v)-(vi) are met. 

The theorem now follows by combining \cite[Theorem 6.1.1]{KuY03} with the characterization of the solution of the projected ODE (\ref{eq-pode}) given by Lemma~\ref{lma-pode}, using the fact that under Assumptions~\ref{cond-bpolicy} and \ref{cond-A}, the matrix $\bA$ is negative definite (Prop.~\ref{prp-ndef}).
\end{proof}

\section{Negative Definiteness of the Matrix $\bA$} \label{appsec-ndef}

In this appendix we prove a necessary and sufficient condition (Prop.~\ref{prp-ndef} below) for the matrix $\bA$ associated with ETD($\lambda$) to be negative definite. 
Recall from Eqs.~(\ref{eq-m})-(\ref{eq-Ab}) that
$$ \bA = - \Fe^\top \, \bM (I -  \PL)  \, \Fe $$
where $\Fe$ is the feature matrix with full column rank, 
$\PL$ is a substochatic matrix, and $\bM$ is a nonnegative diagonal matrix with its diagonal, $diag(\bM)$, given by
$$ diag(\bM) = \bi^\top (I - \PL)^{-1}, \qquad \bi^\top = \big( \d(1) \, \i(1), \, \, \ldots, \,\, \d(N) \, \i(N) \big).$$
Here Assumption~\ref{cond-bpolicy} is in force and ensures that $(I - \PL)^{-1}$ exists and $\d(s) > 0$ for all $s \in \S$. 

The negative definiteness of $\bA$ is important for the a.s.\ convergence of ETD($\lambda$). 
It is known to hold if $\i(s) > 0$ for all $s \in \S$ \citep{SuMW14}, and it is also known that 
in general, $\bA$ is always negative semidefinite 
for nonnegative $\i(\cdot)$. Our result in this appendix will yield a stronger conclusion: $\bA$ is negative definite whenever it is nonsingular. 

In what follows, we first include a proof of the negative semidefiniteness of $\bA$ just mentioned, for completeness (see Prop.~\ref{prp-ndef0}). 
We then give explicitly a condition on the approximation subspace which we will prove to be equivalent to the nonsingularity/negative definiteness of $\bA$ (Prop.~\ref{prp-ndef}).
We also show, by specializing this subspace condition, that if those states $s$ of interest (i.e., $\i(s) >0$) are represented by features $\phi(s)$ that are rich enough, then $\bA$ can be made negative definite, without knowledge of the model (see Cor.~\ref{cor-ndef}, Remark~\ref{rmk-seminorm2}). 
In addition, we discuss the connection of this subspace condition to seminorm projections, and show that when $\bA$ is nonsingular, the ETD($\lambda$) solution can be viewed as the solution of a projected Bellman equation involving a seminorm projection (see Remark~\ref{rmk-seminorm1}).

\subsection{Preliminaries}

First, recall that the matrix $\bA$ is said to be \emph{negative definite} if there exists $c > 0$ such that 
$$y^\top \bA y \leq - c \, \| y\|_2^2, \qquad \forall \,  y \in \rn,$$
and \emph{negative semidefinite} if $c=0$ in the preceding inequality.
The negative definiteness of $\bA$ is equivalent to that of the symmetric matrix
$$ \bA + \bA^\top = - \Fe^\top \Big( \bM (I - \PL) + (I - \PL)^\top \bM \Big) \, \Fe.$$
Similarly to \citep{Sut88,SuMW14}, our analysis will focus on the $N \times N$ symmetric matrix 
$$ \mF = \bM (I - Q) + (I - Q)^\top \bM$$ 
for the substochastic matrix $Q = \PL$ and the nonnegative diagonal matrix $\bM$ as given above.
We will use a theorem from \cite[Cor.\ 1.22, p.\ 23]{Var00}, according to which a symmetric real matrix with positive diagonal entries is positive definite if it is strictly diagonally dominant or irreducibly diagonally dominant. Note that by definition, $\mF$ is \emph{irreducibly diagonally dominant} if $\mF$ is irreducible
\footnote{A symmetric matrix $\mF$ is \emph{irreducible} if it corresponds to a connected (undirected) graph when the indices are viewed as the nodes of the graph, and the nonzero entries of $\mF$ are viewed as edges of the graph.}
and satisfies the following diagonally dominant conditions for every row of $\mF$, with strict inequality holding for at least one row:
$$   | \mF_{ss} | \geq \sum_{\bar s \not=s} | \mF_{s \bar s} |, \qquad s = 1, \ldots, N,     $$
whereas $\mF$ is \emph{strictly diagonally dominant} if it satisfies the above inequalities strictly for all rows.

We now give a proof of the fact that $\bA$ is always negative semidefinite, as mentioned at the beginning. This result is due to \citep{SuMW14}. 

Regarding notation, in the proofs below, for $v \in \re^N$, we write $v(s)$ for the $s$-th entry of $v$,
and for an expression $H$ that results in a vector in $\re^N$, we write $(H)(s)$ for the $s$-th entry of that vector. 
For an expression $H$ that results in an $N \times N$ matrix, we write $[H]_{s\bar s}$ for its $(s,\bar s)$-th element.
We write $\0$ for a zero vector in any Euclidean space.
 
\begin{prop} \label{prp-ndef0}
Let Assumption~\ref{cond-bpolicy} hold. Then $\bA$ is always negative semidefinite, and it is negative definite if $\i(s) > 0$ for all $s \in \S$.
\end{prop}

\begin{proof}
We show that if $\i(s) > 0$ for all $s \in \S$, then $\mF$ is strictly diagonally dominant, and hence positive definite; and that if $\i(s) \geq 0$ for all $s \in \S$, then $\mF$ is positive semidefinite. Since $\Fe^\top \mF \Fe = - \bA - \bA^\top$, the conclusions about $\bA$ will then follow.

Let $\J = \{ s \in \S \mid \i(s) = 0 \}$. Suppose $\J = \emptyset$. By definition $\bM_{ss} = \big(\bi^\top (I - Q)^{-1} \big)(s)$. 
Using this together with the fact that $Q$ is substochastic, by a direct calculation as in \citep{SuMW14}, we have that for each $s \in \S$,
\begin{align}
  \mF_{ss} - \sum_{\bar s\not=s} | \mF_{s \bar s} | & = \bM_{ss} \cdot \left( 1 - \sum_{\bar s=1}^N Q_{s \bar s} \right)  + \sum_{\bar s=1}^N \bM_{\bar s \bar s} \cdot \big[ I - Q\big]_{\bar s  s}  \label{eq-prf-pdef0a} \\
        & \geq \, 0 +  \big( \bi^\top (I - Q)^{-1} \cdot (I - Q) \big) (s) \notag \\
         & = \, 0 + \bi(s) \label{eq-prf-pdef0b} \\
         &  > \, 0, \notag
\end{align}
where in the last strict inequality, we used the fact that $\i(s) > 0$ implies $\bi(s) > 0$ under Assumption~\ref{cond-bpolicy}(ii).
This shows that $\mF$ is strictly diagonally dominant with positive diagonal entries, and hence positive definite by \cite[Cor.\ 1.22]{Var00}.

Consider now the case $\J \not= \emptyset$. For all $s \in \J$, perturb $\i(s)$ to $\delta > 0$, and denote by $\mF_\delta$ the matrix $\mF$ corresponding to the perturbed $\i(\cdot)$. Then $\mF_\delta$ is positive definite by the preceding proof. So for the original $\mF$, by continuity, $\mF = \lim_{\delta \to 0} \mF_{\delta}$ is positive semidefinite.
\end{proof}

\subsection{Main Result}

We now give the main result of this section. It expresses the negative definiteness condition on $\bA$ explicitly in terms of a condition on the approximation subspace $E$ (the column space of $\Fe$), and it establishes the equivalence between the nonsingularity of $\bA$ and the negative definiteness of $\bA$. 

\begin{prop} \label{prp-ndef}
Let Assumption~\ref{cond-bpolicy} hold, and let $\J_0 = \{ s \in \S \mid \bM_{ss}= 0\}$. Suppose the approximation subspace $E\subset \re^{N}$ is such that 
\begin{equation} \label{cond-seminorm-a}
 v \in E \ \ \text{and} \ \  v(s) = 0,  \ \forall \, s \not\in \J_0  \qquad \Longrightarrow \qquad  v = \0.
\end{equation}  
Then the matrix $\bA$ is negative definite. Furthermore, $\bA$ is nonsingular if and only if the condition~(\ref{cond-seminorm-a}) holds.
\end{prop}

The corollary below gives a sufficient condition (\ref{cond-seminorm}) for $\bA$ being negative definite, which can be fulfilled without knowledge of the model, as we will elaborate in Remark~\ref{rmk-seminorm2}. This corollary is a direct consequence of the preceding proposition, and follows from the observation that since $\i(s) > 0$ implies $\bM_{ss} > 0$, the condition (\ref{cond-seminorm}) implies the condition (\ref{cond-seminorm-a}) in Prop.~\ref{prp-ndef}. 

\begin{cor} \label{cor-ndef}
Let Assumption~\ref{cond-bpolicy} hold, and let $\J = \{ s \in \S \mid \i(s) = 0\}$. Suppose the approximation subspace $E\subset \re^{N}$ is such that 
\begin{equation} \label{cond-seminorm}
 v \in E \ \ \text{and} \ \  v(s) = 0,  \ \forall \, s \not\in \J  \qquad \Longrightarrow \qquad  v = \0.
 \end{equation}
Then the matrix $\bA$ is negative definite.
\end{cor}

We now proceed to prove Prop.~\ref{prp-ndef}. Roughly speaking, the method of proof is to decompose the matrix $\mF$ into irreducible diagonal blocks and use, among others, the theorem \cite[Cor.\ 1.22, p.\ 23]{Var00} on irreducibly diagonally dominant matrices mentioned earlier.

In the two technical lemmas that follow, we let the matrix $\mF$ and the nonnegative diagonal matrix $\bM$ take a slightly more general form: 
$$ \mF = \bM (I - Q) + \big( \bM (I - Q)\big)^\top, \qquad diag(\bM) = \bi^\top \, (I - Q)^{-1}, $$
where $Q$ is a substochastic matrix (not necessarily $\PL$), and $\bi$ is a nonnegative vector (for notational simplicity, we keep using $\bi$ instead of introducing new notation). 

\begin{lem} \label{lma-pdef-a}
Suppose the matrix $(I - Q)$ is invertible. 
Then the $s$-th diagonal entry $\bM_{ss} = 0$ if and only if the $s$-th row and $s$-th column of $\mF$ contain all zeros.
\end{lem}

\begin{proof}
We have $\mF = \bM (I - Q) + \big( \bM (I - Q)\big)^\top$.  Suppose $s$ is a state with $\bM_{ss} \not = 0$. Then the $s$-th row of the matrix $M(I-Q)$ is nonzero (because the $s$-th row of $I-Q$ is nonzero, given that $(I-Q)^{-1}$ exists). The nonzero entries of this row cannot be canceled out by the corresponding entries from the $s$-th row of $\big( \bM (I - Q)\big)^\top$, because $Q$ is a substochastic matrix and $\bM$ is nonnegative. Therefore, the $s$-th row of $\mF$ must also be nonzero. This proves the ``if'' part. 

For the ``only if'' part, suppose $s$ is a state with $\bM_{ss}=0$.  Then the $s$-th row of the matrix $\bM(I-Q)$ contains all zeros, so, since $\mF = \bM (I - Q) + \big( \bM (I - Q)\big)^\top$ and is symmetric, to prove the ``only if'' part, we only need to show that the $s$-th column of $\bM(I-Q)$ is a zero column. We prove this by contradiction. 

Suppose for some state $\bar s \not=s$, the $(\bar s, s)$-entry of the matrix $\bM(I-Q)$ is nonzero. Then using the definition of $\bM_{\bar s \bar s}$, this entry can be expressed as
$$   \M_{\bar s \bar s} \cdot \big[ I-Q \big]_{\bar s s} = -  \big( \bi^\top \, (I - Q)^{-1} \big)(\bar s) \cdot Q_{\bar s s} \not=0,$$
which, in view of the equality $(I-Q)^{-1} = \sum_{k \geq 0} Q^k$ and the nonnegativity of $Q$, implies that
$$  \big(\bi^\top \, Q^k \big) (\bar s) \cdot Q_{\bar s s} > 0 \quad \text{for some } k \geq 0.$$
This in turn implies that for the state $s$,
$$   \big(\bi^\top \, Q^k \big) (s) > 0 \quad \text{for some } k \geq 0,   $$ 
and hence 
$$\bM_{ss} = \big( \bi^\top \, (I - Q)^{-1} \big)( s) \geq \big(\bi^\top \, Q^k \big) (s) > 0,$$ 
contradicting the assumption $\bM_{ss}=0$. Thus the $s$-th column of $\bM(I-Q)$ must be a zero column. 
\end{proof}

\begin{lem} \label{lma-pdef-b}
Suppose that the matrix $(I - Q)$ is invertible and the matrix $\mF$ is irreducible. Then the diagonal entries of $\bM$ must be positive, and $\mF$ is irreducibly diagonally dominant with positive diagonal entries, and hence positive definite.
\end{lem} 

\begin{proof}
If $s$ is a state with $\bM_{ss}=0$, by Lemma~\ref{lma-pdef-a}, the $s$-th row and $s$-th column of $\mF$ would contain all zeros, which cannot happen if $\mF$ is irreducible. Thus $\bM_{ss} > 0$ for all $s \in \S$. 

We have calculated in the proof of Prop.~\ref{prp-ndef0} [cf.\ Eqs.~(\ref{eq-prf-pdef0a})-(\ref{eq-prf-pdef0b})] that for nonnegative $\i(\cdot)$,
\begin{align}
  \mF_{ss} - \sum_{\bar s\not=s} | \mF_{s \bar s} | & = \bM_{ss} \cdot \left( 1 - \sum_{\bar s=1}^N Q_{s \bar s} \right)  + \sum_{\bar s=1}^N \bM_{\bar s \bar s} \cdot \big[ I - Q\big]_{\bar s  s} \,  \geq 0 \notag 
\end{align}
for all rows $s$. The strict inequality $\mF_{ss} - \sum_{\bar s \not=s} | \mF_{s \bar s} |  > 0$ must hold for some $s$. To see this, note that the invertibility of $(I - Q)$ implies that $ 1 - \sum_{\bar s =1}^N Q_{s \bar s}  > 0$ for some $s$, which together with $\bM_{ss} > 0$ implies that the first term in the right-hand side above, $\bM_{ss} \cdot \left( 1 - \sum_{\bar s=1}^N Q_{s \bar s} \right)$, must be positive for at least one row $s$, whereas the second term in the right-hand side above equals $\bi(s)\geq 0$ [cf.\  Eqs.~(\ref{eq-prf-pdef0a})-(\ref{eq-prf-pdef0b})]. Since $\mF$ is irreducible by assumption, this proves that $\mF$ is irreducibly diagonally dominant. 

Finally, since $Q$ is substochastic and $(I - Q)^{-1}$ exists, the diagonals of $I - Q$ must be positive. The diagonals of $\bM$ are also positive, as proved earlier. Thus the diagonal entries $\mF_{ss} >0$ for all rows $s$. It then follows from \cite[Cor.\ 1.22]{Var00} that $\mF$ is positive definite.
\end{proof}

We are now ready to prove Prop.~\ref{prp-ndef}. Regarding notation, in the proof, if $\mF_1, \mF_2, \ldots, \mF_L$ are $L$ square matrices (which can have different sizes), 
we will write $diag\big(\mF_1, \mF_2, \ldots, \mF_L\big)$ for the block-diagonal matrix that has $\mF_k$ as its $k$-th diagonal block. However, for a single square matrix $\mF_1$, we will keep using $diag(\mF_1)$ to mean the diagonal of $\mF_1$.
\medskip

\begin{proofof}{Prop.~\ref{prp-ndef}}
By Assumption~\ref{cond-bpolicy}(i), $(I - \P \Gm)^{-1}$ exists, which implies that for the substochastic matrix $Q = \PL$ [cf.\ Eq.~(\ref{eq-bellman-def})], $(I - Q)^{-1}$ also exists. So the matrices $\bM$, $\bA$ and $\mF$ are well defined. 
By reordering the states if necessary, we can arrange $\mF$ into a block-diagonal matrix with $L$ blocks, 
\begin{equation} \label{eq-dcmp-G}
  \mF = diag\Big( \mF^{(1)}, \, \ldots, \, \mF^{(L-1)}, \, \mF^{(L)} \Big) 
\end{equation}  
such that: 
\begin{enumerate}
\item[(i)] for each $\ell=1, \ldots, L-1$, the $\ell$th-block $\mF^{(\ell)}$ is irreducible; and 
\item[(ii)] the $L$-th block $\mF^{(L)}$ is a zero matrix (if $\mF$ does not have a zero block, we will treat $\mF^{(L)}$ as a matrix of size zero, and this will not affect the proof below).
\end{enumerate}
Note that by Lemma~\ref{lma-pdef-a}, 
the row/column indices associated with the zero block $\mF^{(L)}$ are exactly those in the set
\begin{equation} 
  \J_0 = \{ s \in \S \mid \bM_{ss} = 0 \}. \notag
\end{equation} 
Since the condition~(\ref{cond-seminorm-a}) rules out the case $\J_0=\S$, $\mF$ cannot be a zero matrix, so it must have at least one irreducible block.

We prove next that the matrix $Q$ has the following structure, matching the block-diagonal structure of $\mF$:
\begin{equation} \label{eq-dcmp-Q}
     Q = \left[ \begin{matrix}    
         Q^{(1)}  &   &   &  &  \\
            & Q^{(2)}   &   &  &  \\*[-2mm]
                  & &    \line(2,-1){18}  & & \\
                     & & & Q^{(L-1)} &  \\
                     * & * &  \cdots  & * \ & *
         \end{matrix}  
      \right]
\end{equation} 
where the blocks $Q^{(\ell)}$, $\ell \leq L-1$, on the diagonal correspond to the blocks $\mF^{(\ell)}$, $\ell \leq L-1$, on the diagonal of $\mF$, the unmarked blocks contain all zeros, and the $*$-blocks can have both zeros and positive entries. 

To prove Eq.~(\ref{eq-dcmp-Q}) by contradiction, suppose it does not hold. This means that there must exist two states $s \not= \bar s$ with $Q_{s \bar s} > 0$, but the entry $Q_{s \bar s}$ lies inside an unmarked block of the matrix on the right-hand side of Eq.~(\ref{eq-dcmp-Q}). This position of $Q_{s \bar s}$ implies $\mF_{s \bar s} = 0$, which is possible only if $\bM_{ss} = 0$ (otherwise, $Q_{s \bar s} \not=0$ would force $\mF_{s \bar s} \not= 0$). But if $\bM_{ss}=0$, $s \in \J_0$, which is the set of indices associated with the last zero block, as shown earlier.
Consequently, the entry $Q_{s \bar s}$ cannot lie inside an unmarked block as we assumed. This contradiction proves that Eq.~(\ref{eq-dcmp-Q}) must hold.

From the structure of $Q$ shown in~(\ref{eq-dcmp-Q}), it follows that $(I - Q)^{-1}$ has the same structure: 
\begin{equation} \label{eq-dcmp-IQ}
    \big( I - Q \big)^{-1} = \left[ \begin{matrix}    
        \big(I - Q^{(1)} \big)^{-1}   &   &   &  &  \\
            &  \big(I - Q^{(2)} \big)^{-1} &   &  &  \\*[-2mm]
                  & &    \line(2,-1){18}  & & \\
                     & & &  \big(I - Q^{(L-1)} \big)^{-1}&  \\
                     * & * &  \cdots  & * \ & *
         \end{matrix}  
      \right].
\end{equation} 
Since $\mF = \bM (I - Q) + (I - Q)^\top \bM$, 
Eqs.~(\ref{eq-dcmp-G}), (\ref{eq-dcmp-Q}) and (\ref{eq-dcmp-IQ}) together imply that for each $\ell \leq L-1$, the matrix $\mF^{(\ell)}$ can be expressed as
$$ \mF^{(\ell)} = \bM^{(\ell)} \big(I - Q^{(\ell)} \big) + \big(I - Q^{(\ell)} \big)^\top \bM^{(\ell)},$$
where $\bM^{(\ell)}$ is the  $\ell$-th diagonal block in the corresponding decomposition of $\bM$ as 
$$\bM = diag\big(\bM^{(1)}, \ldots, \bM^{(L)}  \big),$$ 
and if we decompose the vector $\bi$ similarly as $\bi = \big(\bi^{(1)}, \ldots, \bi^{(L)} \big)$, then for each $\ell \leq L - 1$, the diagonal block
$\bM^{(\ell)}$ has its diagonal entries given by
$$diag\big(\bM^{(\ell)}\big) = \big( \bi^{(\ell)} \big)^\top \big(I - Q^{(\ell)} \big)^{-1}, \qquad \ell \leq L - 1.$$ 
In the above expression, we also used the fact $\bi^{(L)} = \0$, which is implied by $\bM^{(L)}$ being a zero matrix (which we showed at the beginning of this proof).
\footnote{Using the expression $(I - Q)^{-1} = \sum_{k\geq0} Q^k$, it can be seen from the definition of $\bM_{ss}$ that $\bM_{ss} \geq \bi(s)$. Therefore, $\bM_{ss} = 0$ implies that $\bi(s) = 0$.}

We now apply Lemma~\ref{lma-pdef-b} to each irreducible block $\mF^{(\ell)}$, $\ell \leq L - 1$ (with $\bM = \bM^{(\ell)}$ and $Q=Q^{(\ell)}$, a substochastic matrix). This yields that each of these $\mF^{(\ell)}$ is positive definite, and consequently, the block-diagonal matrix
$$\hat \mF = diag \Big( \mF^{(1)}, \, \ldots, \, \mF^{(L-1)} \Big)$$ 
is positive definite.

Finally, we prove the statement of the proposition. 
For the block-diagonal decomposition of $\mF$ as $\mF = diag (\hat \mF, \mF^{(L)})$, write a point $y \in \re^{N}$ correspondingly as $y = (y_1, y_0)$. I.e., the indices of the components of $y_0$ are those in $\J_0= \{ s \in \S \mid \bM_{ss} = 0 \}$, and the dimension of $y_1$ is $\hat N = N - | \J_0 |$.

Since $\hat \mF$ is positive definite, there exists some $c > 0$ such that
\begin{equation} \label{eq-prf-pdef2}
     {y_1}^\top \hat \mF \, y_1 \geq c \, \| y_1 \|_2^2,  \qquad \forall \,   y_1 \in \re^{\hat N}.
\end{equation}
Consider a point $y =(y_1, y_0) \in E$ with $y_1 = \0$. Then $y_0=\0$ by the assumption (\ref{cond-seminorm-a}). 
Since $E$ is a subspace, this implies that there exists some constant $\delta > 0$ such that 
\begin{equation} \label{eq-prf-pdef3}
   \inf_{y \in E, \, \| y \|_2 = 1} \| y_1 \|_2  \,  \geq \delta.
\end{equation}   
Using Eqs.~(\ref{eq-prf-pdef2})-(\ref{eq-prf-pdef3}), we have
\begin{equation} \label{eq-prf-pdef4}
  \inf_{y \in E, \, \| y \|_2 = 1} y^\top \mF \, y \, = \inf_{y \in E, \, \| y \|_2 = 1} {y_1}^\top \hat \mF \, y_1 \,  \geq  \, \inf_{y \in E, \, \| y \|_2 = 1} c \,   \| y_1\|_2^2 \, \geq c \, \delta^2 > 0.
\end{equation}  
Since $E$ is the column space of $\Fe$ and $\Fe$ has linearly independent columns by definition, the inequality (\ref{eq-prf-pdef4}) establishes that the matrix $\Fe^{\top} \mF \Fe = - \bA - \bA^\top$ is positive definite, and consequently, $\bA$ is negative definite.

The preceding proof also shows that $\bA$ is nonsingular if the condition~(\ref{cond-seminorm-a}) holds. To complete the proof, let us assume that the condition~(\ref{cond-seminorm-a}) does not hold and show that $\bA$ must be singular. We will use the structure of the matrix $\bM(I - Q)$ revealed in the preceding analysis to proof this. Decompose $\Fe$ into two blocks as 
$$\Fe = \left[ \begin{matrix}  
    \Fe_1 \\
    \Fe_0
     \end{matrix}  
      \right] $$
where $\Fe_0$ consists of those rows of $\Fe$ whose indices are in $\J_0= \{ s \in \S \mid \bM_{ss} = 0 \}$. 
Denote
$$ \hat M = diag\big(\bM^{(1)}, \ldots, \bM^{(L-1)}  \big), \qquad \ \hat Q = diag\big(Q^{(1)}, \ldots, Q^{(L-1)} \big),$$
which are the sub-matrices of $\bM$ and $Q$, respectively, obtained by deleting the rows/columns whose indices are in $\J_0$.
From the structure of the matrices $Q, (I - Q)^{-1}$ and the corresponding expression of $\bM^{(\ell)}, \ell \leq L-1$, that we showed earlier, it follows that the matrix $ \bA = - \Fe^\top \bM (I -  Q)  \Fe $ is indeed given by
$$ \bA = - \Fe_1^\top \, \hat M (I - \hat Q) \, \Fe_1.$$
Now if the condition (\ref{cond-seminorm-a}) does not hold, then there exists $y =(y_1, y_0) \in E$ such that $y_1=\0$ and $y_0 \not=\0$. 
Expressing $y$ in terms of $\Fe$, we have $y_1 = \Fe_1 x = 0$ and $y_0 = \Fe_0 x \not=0$ for some nonzero $x \in \rn$. 
This implies $\text{rank}(\Fe_1) < n$, so using the preceding expression of $\bA$, we have $\text{rank}(\bA) < n$ and hence $\bA$ is singular.
\end{proofof}

Finally, we make two remarks on the conditions (\ref{cond-seminorm-a}) and (\ref{cond-seminorm}) in Prop.~\ref{prp-ndef} and Cor.~\ref{cor-ndef}.

\begin{rem}[Seminorm projection] \label{rmk-seminorm1}
\rm Using seminorm projections to formulate the projected Bellman equations associated with TD methods is introduced in \citep{yb-bellmaneq}. There, conditions of the form~(\ref{cond-seminorm-a}) or (\ref{cond-seminorm}) are used to define a projection on the approximation subspace with respect to a seminorm.
We can use this formulation here to interpret the solution of ETD($\lambda$) and ELSTD($\lambda$). 
Specifically, define a weighted Euclidean seminorm $\| \cdot \|_\bM$ on $\re^{N}$, using $diag(\bM)$ as the weights, as 
$$ \| v \|^2_{\bM} = \textstyle{\sum}_{s \in \S}  \, \bM_{ss} \cdot v(s)^2.$$ 
Condition~(\ref{cond-seminorm-a}) ensures     
that the projection $\Pi_{\bM}$ onto $E$ with respect to the seminorm $\| \cdot \|_{\bM}$ is well-defined and has the matrix representation
$$ \Pi_{\bM} = \Fe \, \big( \Fe^\top \bM \Fe \big)^{-1} \, \Fe^\top  \bM$$
\citep[cf.][Sec.\ 2.1]{yb-bellmaneq}. So by Prop.~\ref{prp-ndef} and the convergence results of this paper, when $\bA$ is nonsingular, 
ETD($\lambda$) and ELSTD($\lambda$) solve in the limit the projected Bellman equation
$$ v = \Pi_{\bM} \big( \rl + \PL \, v \big).$$
The relation between the solution $v = \Fe \w^*$ of this equation and the desired value function $v_\pi$, in particular, the approximation error, can be analyzed then, using the oblique projection viewpoint \citep{bruno-oblproj} (for details, see also \citep{yb-bellmaneq}).
\end{rem}

\begin{rem}[Equivalent conditions in terms of features] \label{rmk-seminorm2}
\rm 
The condition~(\ref{cond-seminorm-a}) can be paraphrased in terms of the features $\fe(s)$ as follows:
\begin{equation} \label{eq-feat0}
  \forall \, s \in \S \  \text{with} \ \bM_{ss} = 0, \quad \fe(s)  \in span \big\{ \fe(\bar s) \, \big| \, \bar s \in \S \ \text{and} \ \bM_{\bar s \bar s} > 0 \big\}; 
\end{equation}  
namely, from those states with positive emphasis weights $\bM_{\bar s \bar s} > 0$, $n$ linearly independent feature vectors can be found.
Similarly, the condition~(\ref{cond-seminorm}) can be paraphrased as:
\begin{equation} \label{eq-feat}
  \forall \, s \in \S \  \ \text{with} \ \i(s) = 0, \quad \fe(s)  \in span \big\{ \fe(\bar s) \, \big| \, \bar s \in \S \ \text{and} \ \i(\bar s) > 0 \big\};
\end{equation}  
namely, from the states with positive interest weights, $n$ linearly independent feature vectors can be found.
This shows that even without knowing $\P$ and $\bM$, by designing a rich enough set of features for states of interest beforehand, we can ensure the sufficient condition~(\ref{cond-seminorm}) for the nonsingularity and negative definiteness of the matrix $\bA$. 

Conditions like (\ref{eq-feat0}), (\ref{eq-feat}) [or equivalently, (\ref{cond-seminorm-a}), (\ref{cond-seminorm})] are naturally satisfied in the case where the approximate values of the policy $\pi$ at certain states (e.g., those states $s$ with $\bM_{ss} = 0$ or $\i(s) = 0$) are interpolated or extrapolated from the approximate values of $\pi$ at some other ``representative'' states, based on the ``proximity'' of the former states to the representative ones.
\end{rem}

\end{document}